\documentclass[11pt]{article}

\usepackage{amsmath}
\usepackage{amsfonts}
\usepackage{bm}
\usepackage{authblk}
\usepackage{graphicx}
\usepackage{url}
\usepackage{listings}
\usepackage{xcolor}

\usepackage{times}
\usepackage{helvet}
\usepackage{courier}
\usepackage{subfigure}
\usepackage{algorithm, algorithmic}
\usepackage{amssymb}
\usepackage{amsthm}
\usepackage{wrapfig}
\usepackage{multirow}

\newtheorem{thm}{Theorem}
\newtheorem{lem}[thm]{Lemma}
\newtheorem{prop}[thm]{Proposition}

\definecolor{mygreen}{rgb}{0,0.6,0}
\definecolor{mygray}{rgb}{0.5,0.5,0.5}
\definecolor{mymauve}{rgb}{0.58,0,0.82}

\begin{document}

\title{Unifying the Stochastic Spectral Descent for Restricted Boltzmann Machines with Bernoulli or Gaussian Inputs}
\author[1]{Kai Fan\thanks{Email: kai.fan@stat.duke.edu}}

\affil[1]{Duke University}

\maketitle

\begin{abstract}
Stochastic gradient descent based algorithms are typically used as the general  optimization tools for most deep learning models. 
A Restricted Boltzmann Machine (RBM) is a probabilistic generative model that can be stacked to construct deep architectures. 
For RBM with Bernoulli inputs, non-Euclidean algorithm such as stochastic spectral descent (SSD) has been specifically designed to speed up the convergence with improved use of the gradient estimation by sampling methods. 
However, the existing algorithm and corresponding theoretical justification depend on the assumption that the possible configurations of inputs are finite, like binary variables. 
The purpose of this paper is to generalize SSD for Gaussian RBM being capable of modeling continuous data, regardless of the previous assumption. 
We propose the gradient descent methods in non-Euclidean space of parameters, via deriving the upper bounds of logarithmic partition function for RBMs based on Schatten-$\infty$ norm.
We empirically show that the advantage and improvement of SSD over stochastic gradient descent (SGD). 
\end{abstract}

\section{Introduction}
\label{sec:intro}

A crucial challenge to many deep learning models involves optimizing non-convex objective functions over continuous high dimensional matrix or tensor spaces, with large scale training dataset. 
The stochastic version of gradient descent or quasi-Newton method is ubiquitously applied to optimize the parameters of these models, intending to find an acceptable local optima. 
However, exploiting the inhomogeneous curvature of loss functions or the corresponding information in their gradients more effectively can greatly speed up the overall inference. 
\cite{carlson2015stochastic} developed stochastic spectral descent (SSD) method by changing the space in which the gradient descent performs to better utilize the natural geometry implied in the Bernoulli Restricted Boltzmann Machine (RBM) \cite{hinton2006fast} by operating in the $l_\infty$ and Schatten-$\infty$ (S-$\infty$) norms \cite{bhatia2013matrix}.
The advent of such customized decent method encourages more specific designed algorithms that allow for optimizing discrete deep modeling, \cite{carlson2015preconditioned,carlson2016stochastic}. 
Figure \ref{fig:ssd_eg} illustrates that SSD tends to find the shortcut path toward the local optimal. 

\begin{figure}[t]
\centering
\includegraphics[width=41mm]{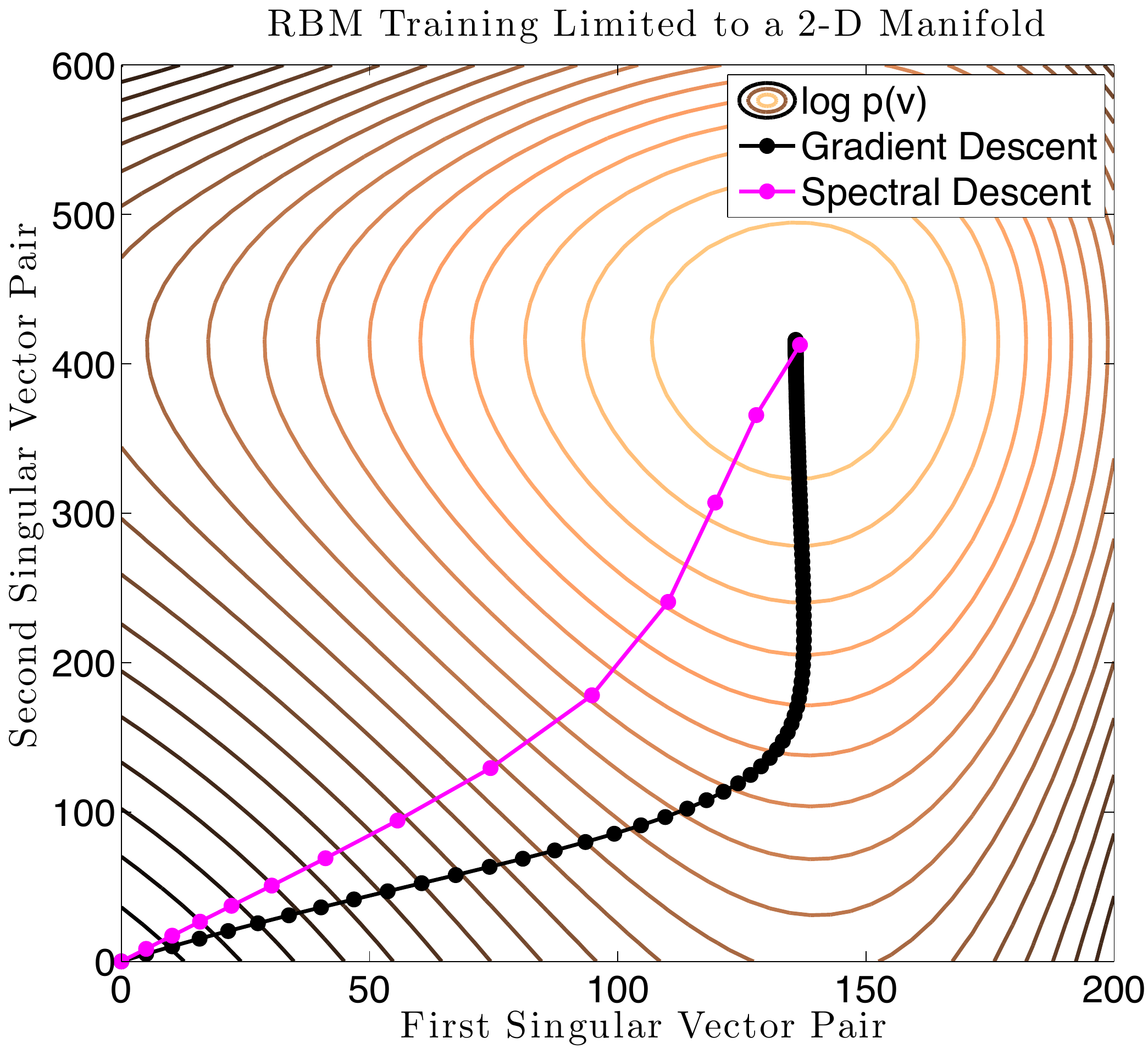}
\includegraphics[width=41mm]{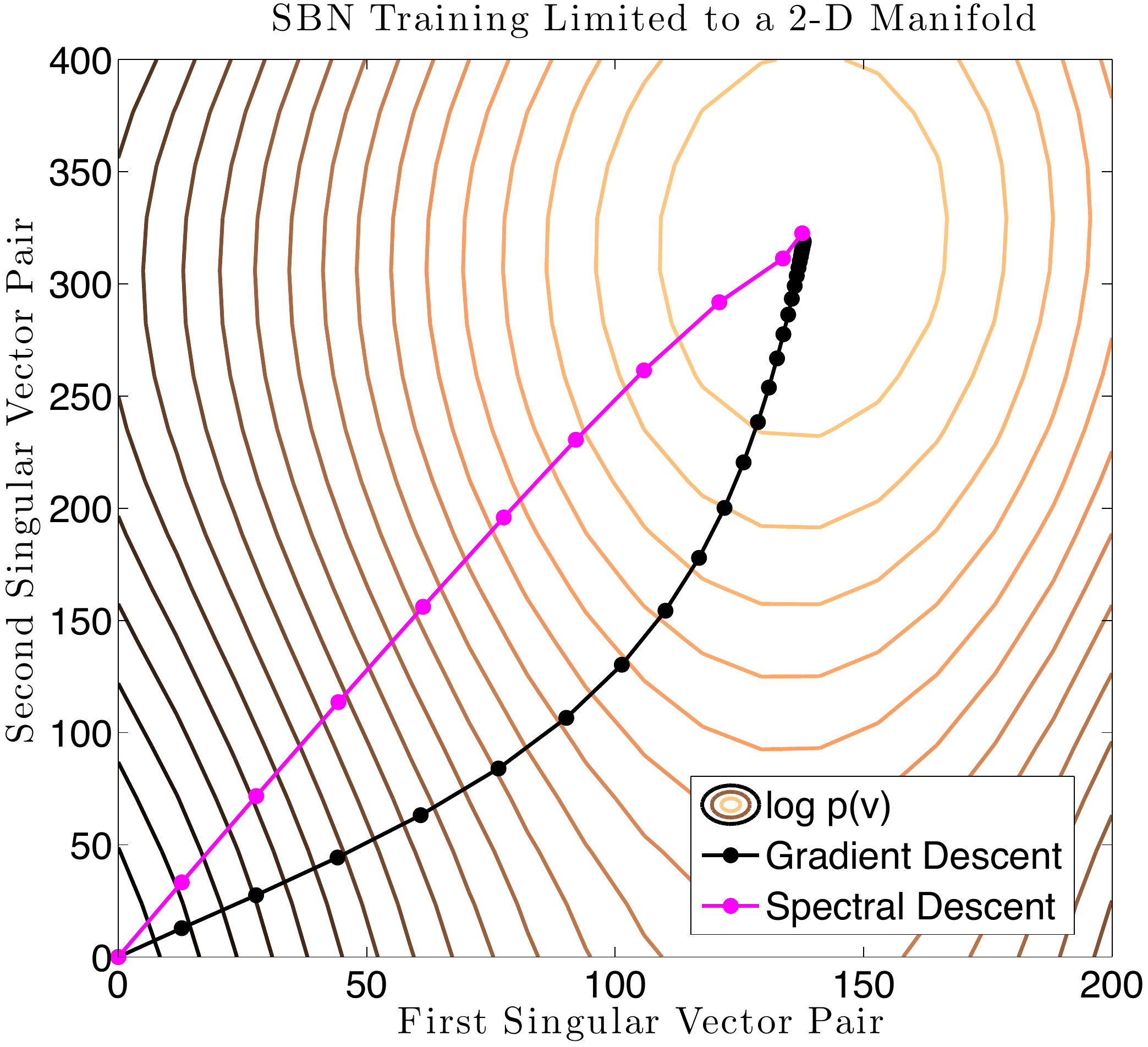}
\caption{Two examples on comparing the optimized path.}
\label{fig:ssd_eg}
\end{figure}

The previous research of SSD is however restrictive and it is theoretically correct for discrete data with finite domain. 
A common trick to modeling real-valued data by RBM is normalizing each input variable to $[0,1]$ interval and considering it as Bernoulli distributed variable \cite{hinton2006reducing}. 
This approach has its own limitation as well since it usually requires the data input to be bounded. 
This means the generalization of RBM for Gaussian distributed input (typically denoted as Gaussian RBM or abbreviated as GRBM) is important. 
Alternatively, we can further generalize the SSD algoirthm that can be used to optimize GRBM.

Along the lines of \cite{carlson2015stochastic}, it is crucially required to derive an upper bound for loss function of GRBM in the sense of spectral norm or infinity norm. 
Since the GRBM will basically introduce another parameter $C^{-1}$ to represent the inverse of covariance matrix for observed variable and modify the energy function to quadratic accordingly, to bound the loss function is not as trivial for GRBM with real valued visible inputs.
To achieve this desideratum, we make an additional but reasonable assumption: the norm of weight matrix is bounded by some constant. 
This assumption is in fact implicitly employed in real deep learning programming. Most of weights regularization tricks are forcing a smaller or bounded norm, such as renormalization or weights decay \cite{hinton2012improving}. 
Under this assumption, we can successfully derive an upper bound of loss function by a local approximation with first order Taylor expansion plusing the error term in $l_\infty$ or Schattern-$\infty$ norm. 
To the end, we can unify the SSD algorithm for RBM with either binary or real valued inputs, being a substitute of vanilla SGD for RBM and able to apply various adaptive preconditioners.

\subsubsection{Related Works}

\cite{carlson2015stochastic} is the most related work of this paper, since our work is its generalization, and a demonstration of its correctness. 
In addition, many stochastic gradient descent algorithms are developed for deep learning problems in recent years, including Adagrad \cite{duchi2011adaptive}, Rmsprop \cite{tieleman2012lecture}, AdaDelta \cite{zeiler2012adadelta}, Adam \cite{kingma2014adam}. 
These methods are all trying to approximate the inverse of Hessian matrix in an efficient way and use it as a preconditioner. 
However, the way they derive the update rule for parameter still falls into the Taylor approximation in the sense of $l_2$ or Frobenius norm, significantly differing from SSD. 
Another work \cite{dauphin2014identifying} mentioned that the pathological curvature can make the algorithms get stuck more iterations for escaping the saddle point, and proposed another upper bound of loss function by adjusting quadratic term of Taylor expansion. 
It basically set the negative eigenvalue of Hessian matrix to be its absolute value and plugged it back to the quadratic term, which indicated it still considers the $l_2$ norm. 
AdaMax \cite{kingma2014adam} is one algorithm considering the $l_\infty$ norm for vector space in a moving average way, which is also different from the S-$\infty$ in matrix space.

As to Gaussian RBM, a typical definition with diagonal covariance matrix is defined in \cite{cho2013gaussian}, and this work focused on the parallel tempering algorithms. 
Our analysis of SSD w.r.t. GRM is not limited to diagonal covariance matrix, although it is more efficient and works well in real implementation. 
In \cite{hinton2010modeling}'s work, it introduced the model called Factorized Third-Order RBM intending to model pixel mean and covariances. 
This model can be considered as one variant of GRBM, while the covariance matrix is presumingly defined as some factorized structure. 
The parameterization of such a covariance matrix is complicated, and difficult to generalize with SSD in the theoretical aspect.

\section{Preliminaries}
\label{sec:pre}

\subsection{RBM with Bernoulli Input}
We briefly review the traditional RBM based on \cite{salakhutdinov2008quantitative}. 
The RBM is a two-layer binary Markov Random Field, where the model includes observed layer units $\mathbf{v} \in \{0,1\}^{N_v}$ and its pairwise connected hidden units $\mathbf{h} \in \{0,1\}^{N_h}$. 
The ``restricted" means there exists a matrix $W\in \mathbb{R}^{N_v\times N_h}$ to capture the symmetric interaction cross two layers but no self interaction within the layer, thus defining the energy function of the state $\{\mathbf{v}, \mathbf{h}\}$ as 
\begin{align}\label{eq:brbm}
E(\mathbf{v},\mathbf{h}; \bm\theta) = - \mathbf{v}^\top W \mathbf{h} - \mathbf{v}^\top \mathbf{b} - \mathbf{h}^\top \mathbf{a}
\end{align} 
where all parameters $\bm\theta=\{W, \mathbf{b}, \mathbf{a}\}$. 
The probability density function of observed variable is
\begin{align}
p(\mathbf{v};\bm\theta) &= \frac{1}{Z(\bm\theta)}\sum_{\mathbf{h}} \exp\left( E(\mathbf{v},\mathbf{h}; \bm\theta) \right) \label{eq:pv}\\
Z(\bm\theta) &= \sum_{\mathbf{v}} \sum_{\mathbf{h}} \exp\left( E(\mathbf{v},\mathbf{h}; \bm\theta) \right) \label{eq:binary_Z}
\end{align}
where $Z(\bm\theta)$ is the partition function. The objective is to minimize the negative log-likelihood over the dataset $\{\mathbf{v}_{n=1}^N\}$, 
\begin{align} \label{eq:loss}
\mathcal{L}(\bm\theta) = - \frac{1}{N} \sum_{n=1}^N \log p(\mathbf{v}_n;\bm\theta) = f(\bm\theta) + g(\bm\theta)
\end{align}
Notice that we can decompose the objective function into the summation of two parts: the data independent log-partition function $f(\bm\theta) = \log Z(\bm\theta)$, and the unnormalized data negative log-likelihood $g(\bm\theta) = - \frac{1}{N} \sum_{n=1}^N \sum_{\mathbf{h}} \exp\left( E(\mathbf{v}_n,\mathbf{h}; \bm\theta) \right)$.  

The derivative of the objective function w.r.t. $W$ is:
\begin{align}
\nabla_W \mathcal{L}(\bm\theta) = \mathbb{E}_p[\mathbf{v}\mathbf{h}^\top] - \mathbb{E}_e[\mathbf{v}\mathbf{h}^\top]
\end{align}
where $\mathbb{E}_p[\cdot]$ represents the expectation w.r.t. the model defined  distribution, and $\mathbb{E}_e[\cdot]$ denotes the empirical distribution defined on dataset. 
However, $\mathbb{E}_p[\cdot]$ involves expectation or summation that cannot be analytically computed. 
In practice parameter learning is performed by following an approximation to gradient of a different objective function, which means $\mathbb{E}_p[\cdot]$ is approximated by the distribution of samples obtained from running the Gibbs sampling of RBM for $k$ steps.
This is called the Contrastive Divergence (CD-k) learning \cite{hinton2002training}. 
$k\rightarrow\infty$ recovers the $\mathbb{E}_p[\cdot]$, however, $k$ is set to 1 or small number in practice. 
Alternatively, Persistent CD \cite{tieleman2008training} is more effective.

\subsection{Schatten-$p$ Norm}

For $p\geq 1$, the $l_p$ norm in finite dimensions for a vector $\mathbf{x}$ is defined as $\|\mathbf{x}\|_p=\left( \sum_i |x_i|^p\right)^{1/p}$. 
As $p\rightarrow\infty$, we obtain the $l_\infty$ norm, $\|\mathbf{x}\|_\infty=\max_i\{|x_i|\}$. 
Suppose $\bm\lambda(X)$ is the vector of singular values of matrix $X\in\mathbb{R}^{m\times n}$, we can similarly define the Schatten-$p$ norm $\|X\|_{S_p}=\|\bm\lambda(X)\|_p$, and $\|X\|_{S_\infty}=\|\bm\lambda(X)\|_\infty$. 
Notice singular values are non-negative, so we can omit the absolute operator for $\lambda_i$.

\subsection{SSD for Bernoulli Inputs}
Many stochastic gradient descent algorithms have been developed for discrete data modeling \cite{carlson2015stochastic,carlson2015preconditioned,carlson2016stochastic}, 
due to the fact that
\begin{align*}
\mathcal{L}(W+\Delta W) \leq \mathcal{L}(W) + \langle \nabla_W \mathcal{L}, \Delta W \rangle + \frac{N_vN_h\|\Delta W\|_{S_\infty}^2}{2} ,
\end{align*}
where inner product $\langle\cdot,\cdot \rangle$ means trace of matrix product. 
In general, if we consider the error term above in $l_2$ norm, the local minimization of this objective function indicates $\Delta W \propto \nabla_W \mathcal{L}$, i.e., in practice the gradient estimated by CD-$k$ can be directly used for traditional gradient descent. 
However, in the sense of $S_\infty$ norm, the induced minimization is no longer to maximize in the direction of the gradient. 
It indicates $\Delta W \propto \|\bm\lambda(W)\|_1 UV^\top$, where $U,V$ are unitary matrix obtained from proceeding SVD to $\nabla_W \mathcal{L}$. 
Analogously, SSD performs parameter update as $W \leftarrow W - \epsilon \Delta W$, where $\epsilon$ is the step size.

\section{Gaussian RBM}

In this section, we will generalize the SSD to handle Gaussian RBM (GRBM). Unlike Bernoulli RBM, the observed variable $\mathbf{v} \in \mathbb{R}^{N_v}$ is assumed to be a real-value vector as Gaussian distributed variables, while the binary latent variable $\mathbf{h}$ remains unchanged. 
The Gaussian RBM we discussed here is defined with the following energy function. 
\begin{align}\label{eq:grbm}
E_G(\mathbf{v},\mathbf{h};\bm\theta) = -\mathbf{v}^\top C^{-1} W\mathbf{h}+\frac{(\mathbf{v}-\mathbf{b})^\top C^{-1}(\mathbf{v}-\mathbf{b})}{2}-\mathbf{h}^\top\mathbf{a}
\end{align}
where $\bm\theta=\{W, \mathbf{b}, \mathbf{a}, C^{-1}\}$. 
Notice that if we also set the hidden variable as Gaussian latent variable, it means that the resulted model would be an undirected version of factor analysis and however, it turns out to be identical to the traditional directed version \cite{marks2001diffusion}. 
If we only change the binary hidden variable of Bernoulli RBM to become following Gaussian distribution, it means that the induced model would be a special case of exponential PCA \cite{mohamed2009bayesian}. 
 
The parameters of energy function include an extra covariance matrix of observed variable, and if $C$ is set to be identity matrix, it will recover to the GRBM defined in \cite{murphy2012machine}. 
Identity matrix is so simple that nothing to learn, thus diagonal covariance matrix \cite{cho2013gaussian} is usually preferred in practice due to its efficient computation. 
More complicated covariance matrix design (e.g., \cite{hinton2010modeling} parametrizes $C$ as a structured factorization) exists as well but beyonds our research area.
 
Since $\mathbf{v}$ is a real-valued vector, the summation in partition function $Z(\bm\theta)$ (Eq. (\ref{eq:binary_Z})) over $\mathbf{v}$ is replaced by integration, thus the log-partition function can be written as
\begin{align} \label{eq:logpf}
f(\bm\theta) = \log\left(\int\sum_{\mathbf{h}}\exp\left( - E_G(\mathbf{v},\mathbf{h};\bm\theta)\right)\mathrm{d}\mathbf{v}\right)
\end{align}
While the definitions of $p(\mathbf{v};\bm\theta)$, $\mathcal{L}(\theta)$ and data negative likelihood function $g(\bm\theta)$ will remain the same formula as Eq. (\ref{eq:pv}, \ref{eq:loss}) in previous section, although their internal computation is completely different. 
According to energy function (\ref{eq:grbm}), the conditional distribution can be represented as,
\begin{align*}
p(\mathbf{v}|\mathbf{h},\bm\theta)&=\mathcal{N}\left(\mathbf{b}+W\mathbf{h},C\right) \\
p(h_k=1|\mathbf{v},\bm\theta)&=\text{Bern}\left( \sigma\left([\mathbf{v}^\top C^{-1} W+\mathbf{a}^\top]_k\right) \right)
\end{align*}
where $[\cdot]_k$ means the $k$-th element of vector. 
The tractable conditional distributions enable the possibility of Gibbs sampling for gradient computation (see supplementary materials for the detailed gradient derivation), and it means that CD-$k$ algorithm for gradients estimation can be applied to GRBM directly. 
However, we want to develop more efficient usage of gradients for parameter updates in the framework of SSD. 
It essentially requires that we can upper bound the loss function in the sense of $l_\infty$ or $S_\infty$ norm. 
Our method is to bound $f(\bm\theta)$ and $g(\bm\theta)$ separately. 

\subsection{Upper Bound of $f(\bm\theta)$}

In this section, we will bound the log-partition function w.r.t. each of $\bm\theta=\{W, \mathbf{b}, \mathbf{a}, C^{-1}\}$ when the other three are fixed or $\bm\theta^{old}$ is given. First, we handle the simplest cases for $\mathbf{a}$. By integrating out continuous variable $\mathbf{v}$, we can rewrite $f(\bm\theta)$ as 
$
f(\bm\theta) = \log\left(\sum_{i=1}^{2^{N_h}}\omega_i\exp(\mathbf{h}_i^\top\mathbf{a}) \right)
$, 
where $\omega_i = (2\pi)^{\frac{N_v}{2}}|C|^{\frac{1}{2}}\exp\left(\mathbf{b}^\top C^{-1} W\mathbf{h}_i+ \frac{(W\mathbf{h}_i)^\top C^{-1} W\mathbf{h}_i}{2} \right)$. 
Notice the bianary vector $\mathbf{h}$ can have $2^{N_h}$ configurations. 
To bound $f(\bm\theta)$ w.r.t. $\mathbf{a}$, we require the following Lemma and use it to obtain the upper bound shown in Theorem \ref{thm:bound_a} (See proof in supplementary materials).
\begin{lem}\label{lem:lse} \cite{carlson2015stochastic} 
Consider log-sum-exp function $lse_{\bm\omega}(\mathbf{x})=\log\sum_{i}\omega_i\exp(x_i)$ where $\bm\omega$ is independent of $\mathbf{x}$, the following inequality $lse_{\bm\omega}(\mathbf{x}+\Delta\mathbf{x}) \leq lse_{\bm\omega}(\mathbf{x}) + \langle \nabla lse_{\bm\omega}(\mathbf{x}), \Delta\mathbf{x}\rangle + \frac{1}{2}\|\Delta\mathbf{x}\|_\infty $ holds. 
\end{lem}
\begin{thm}\label{thm:bound_a} When fixing other parameters, we have
\begin{align}
f(\mathbf{a}+\Delta\mathbf{a}) &\leq f(\mathbf{a}) + \langle \nabla_{\mathbf{a}} f(\mathbf{a}), \Delta\mathbf{a}\rangle + \frac{N_h}{2}\|\Delta\mathbf{a}\|_\infty^2  \label{eq:bound_a}
\end{align}
\end{thm}

In order to bound $f(\bm\theta)$ of w.r.t. other parameters, we make an assumption that the weight matrix $W$ has a finite norm, i.e., $\|W\|_2 \leq R$. 
This assumption is theoretically reasonable and practical in most deep learning implementation ($R$ is usually small, e.g.  $R=\sqrt{15}$ in \cite{hinton2012improving}), since either the weights decay or weights renormalization is to force the weights matrix norm to decrease. 
Thus we can readily bound $f(\bm\theta)$ w.r.t. $\mathbf{b}$ by applying Lemma \ref{lem:lse}. We rewrite $f$ as $
f(\bm\theta) = \log\left(\sum_{i=1}^{2^{N_h}}\omega_i\exp(\mathbf{b}^\top C^{-1}W\mathbf{h}_i) \right)
$, 
where $\omega_i=(2\pi)^{\frac{N_v}{2}}|C|^{\frac{1}{2}}\exp\left(\mathbf{h}_i^\top\mathbf{a}+ \frac{1}{2}(W\mathbf{h}_i)^\top C^{-1}W\mathbf{h}_i \right)$
\begin{thm}\label{thm:bound_b} When fixing other parameters, we have
\begin{align}
\hspace{-1cm}
f(\mathbf{b}+\Delta\mathbf{b}) &\leq f(\mathbf{b}) + \langle \nabla_{\mathbf{b}} f(\mathbf{b}), \Delta\mathbf{b}\rangle + \frac{N_hRr(C)}{2}\|\Delta\mathbf{b}\|_\infty^2, \label{eq:bound_b}
\end{align}
where $r(C)$ is a local constant dependent on the parameter in previous iteration.
\end{thm}

$r(C)$ actually depends on the eigenvalues of $C$, where $1/r(C)$ is between the minimum and maximum of the eigenvalues (See discussion in Supplementary materials). 
For inverse covariance matrix $C^{-1}$, we can derive the following representation by integrating out $\mathbf{v}$.
\begin{align*}
f(\bm\theta) =& \log\left(\sum_{i=1}^{2^{N_h}}\omega_i\exp\left( \left(\mathbf{b}+\frac{1}{2}(W\mathbf{h}_i)\right)^\top C^{-1}W\mathbf{h}_i \right) \right) + \frac{1}{2}\log|C|
\end{align*}
where $\omega_i=(2\pi)^{\frac{N_v}{2}}\exp\left(\mathbf{h}_i^\top\mathbf{a}\right)$. 
However, this function involves log determinant, we need another lemma.
\begin{prop}\label{prop:logdet} 
$-\log|C^{-1}|=\log|C|$.
\end{prop}
Identity in Proposition \ref{prop:logdet} allows to consider $f(\bm\theta)$ w.r.t. $C^{-1}$. Lamma \ref{lem:lse} with assumption on $W$ indicates the upper bound for $f$.
\begin{thm}\label{thm:bound_C} Given parameters in previous iteration, $f(C^{-1}+U) \leq f(C^{-1}) + \langle \nabla_{C^{-1}} f(C^{-1}), U\rangle + \left(\frac{N_h^2R^2}{2}+r(C)\right)\|U\|_{S_\infty}^2 $, 
where $U$ is positive semidefinite, inner product is defined as trace for matrix, and $r(C)$ is a local constant. 
\end{thm}

For weight matrix $W$, we notice that the log-partition function has a canonical form that does not fall into $lse$ function, where $W$ appears as quadratic.
\begin{align*}
f(\bm\theta)=\log\left(\sum_{i=1}^{2^{N_h}}\omega_i\exp\left(\frac{1}{2} s(W)^\top s(W) \right) \right)
\end{align*}
where $\omega_i=(2\pi)^{\frac{N_v}{2}}|C|^{\frac{1}{2}}\exp\left(\mathbf{h}_i^\top\mathbf{a}-\frac{1}{2}\mathbf{b}^\top C^{-1}\mathbf{b}\right)$ and $s(W) = C^{-1/2} (W\mathbf{h}_i+\mathbf{b})$. 
Thus, we prove another lemma.
\begin{lem}\label{lem:lse2} 
Consider log-sum-exp-square function $lse2_{\bm\omega}(\mathbf{x})=\log\sum_{i}\omega_i\exp(x_i^2/2)$ where $\bm\omega$ is independent of $\mathbf{x}$, if the domain of this function satisfies the bound condition $\|\mathbf{x}\|_2\leq r$, then the following inequality $lse2_{\bm\omega}(\mathbf{x}+\Delta\mathbf{x}) \leq lse2_{\bm\omega}(\mathbf{x}) + \langle \nabla lse2_{\bm\omega}(\mathbf{x}), \Delta\mathbf{x}\rangle + \left(\frac{1}{2}+\frac{3r^2}{4} \right)\|\Delta\mathbf{x}\|_\infty $ holds. 
\end{lem}
With the definition $lse2$, we can apply Lemma \ref{lem:lse2} to bound $f(\bm\theta)$ w.r.t. $W$ in the sense of Schatten-$\infty$ norm.
\begin{thm}\label{thm:bound_W} When fixing other parameters, we have
\begin{align}
f(W+U) &\leq f(W) + \langle \nabla_W f(W), U\rangle \nonumber\\
&+ \left(\frac{1}{2}+\frac{3r(R,C,\mathbf{b})^2}{4}\right)N_vN_h\|U\|_{S_\infty}^2 \label{eq:bound_W}
\end{align}
where $r(R,C,\mathbf{b})$ depends on the parameter $C,\mathbf{b}$ in previous iteration.
\end{thm}
Notice that in practice we don't care about each constant before all error terms in the sense of infinity norm. 
Since we always need to adjust the learning rate, the learning rate will absorb these constants and it is unnecessary to compute it.

\subsection{Upper Bound of $g(\bm\theta)$}

In previous section, we derive the upper bound for log-partition function. 
We can also rewrite negative likelihood function $g(\bm\theta)$ as
\begin{align*}
g(\bm\theta) =& \frac{1}{N}\sum_{n=1}^N \left(\frac{1}{2}(\mathbf{v}_n-\mathbf{b})^\top C^{-1}(\mathbf{v}_n-\mathbf{b}) \right. \\
 & \left. -\sum_{k=1}^{N_h}\log\left( 1+\exp\left([\mathbf{v}_n^\top C^{-1} W + \mathbf{a}]_k\right)\right)\right) .
\end{align*}
However, we still need to derive the upper bound for $g$, if we want to bound the loss function $\mathcal{L}(\bm\theta)$. 
Notice that due to the property of negative softplus function $-\log(1+e^x)$, $g(\bm\theta)$ is concave w.r.t. $W$ or $\mathbf{a}$, the inequality $g(\mathbf{a}+\Delta\mathbf{a}) \leq g(\mathbf{a}) + \langle \nabla_{\mathbf{a}} g(\mathbf{a}), \Delta\mathbf{a} \rangle$ naturally holds, so does for $W$. 
In fact, $g(\bm\theta)$ is also concave w.r.t. $C^{-1}$, since the first part is linear to $C^{-1}$ (See details in supplementary materials).

As to $\mathbf{b}$, we notice $g$ is quadratic where the gradient that exceeds 2nd order is vainished, i.e., $g(\mathbf{b}+\Delta\mathbf{b}) - g(\mathbf{b}) = \langle \nabla_\mathbf{b} g(\mathbf{b}), \Delta\mathbf{b}\rangle + \frac{1}{2}\Delta\mathbf{b}^\top C^{-1}\Delta\mathbf{b}$. In addition, $\Delta\mathbf{b}^\top C^{-1}\Delta\mathbf{b} \leq N_v r(C)\|\Delta\mathbf{b}\|_\infty$, where $r(C)$ can be determined by the smallest eigenvalue of $C$ (Details see supplementary materials).

\subsection{Revisiting Stochastic Spectral Descent}

To this end, we have successfully bounded $f(\bm\theta)$ and $g(\bm\theta)$ w.r.t. each parameter in $\bm\theta$. 
Thus, the final loss function $\mathcal{L}(\bm\theta)$ can also be bounded in the sense of infinity norm.
\begin{align*}
\mathcal{L}(\bm\theta_i+\Delta\bm\theta_i) \leq \mathcal{L}(\bm\theta_i) + \langle \nabla_{\bm\theta_i} \mathcal{L}(\bm\theta_i), \Delta\bm\theta_i\rangle + c_i\|\Delta\bm\theta_i\|_\infty^2
\end{align*}
The SSD update strategy is derived from minimizing the surrogate of loss function from the right-hand side of above inequality. 
Since the Schattern-$\infty$ norm mainly depends on the largest eigenvalue, the optimal direction to perform gradient descent is different from traditional stochastic gradient descent (SGD). 
However, if the error term is in $l_2$ norm, the update strategy will recover SGD update. 
Additionally, it indicates the inequality holds for both Bernoulli and Gaussian RBM, and guarantees that SSD can be applied. 
Thus, the general SSD algorithm is summarized in Algorithm \ref{alg:ssd}.
\begin{algorithm}[H]
\caption{Stochastic Spectral Descent for RBM}
\label{alg:ssd}
\begin{algorithmic}[1]
\WHILE{$\bm\theta$ Not Converge}
	\STATE Sample minibatch $\mathbf{v}_{1:B}$;
	\STATE $\{\nabla_W, \nabla_\mathbf{a}, \nabla_\mathbf{b}, \nabla_{C^{-1}}\}\mathcal{L}$ estimated by CD-$k$;
	\STATE $S_\infty$ Update for $W$;
	\STATE $l_\infty$ Update for $\mathbf{a},\mathbf{b}$;
	\STATE Gaussian RBM only: $l_\infty$ Update for $C^{-1}$ if $C^{-1}$ is diagonal else $S_\infty$ Update;
\ENDWHILE
\end{algorithmic}
\end{algorithm} 
\vspace{-10mm}
\begin{minipage}[t]{4cm}
 \begin{algorithm}[H]
    \caption{$S_\infty$ Update}
    \label{alg:SU}
    \textbf{Input} Matrix $X, \nabla_X$ \\
    $[U, \bm\lambda, V]=\text{SVD}(\nabla_X)$ \\
    $X \leftarrow X - \epsilon \|\bm\lambda\|_1 UV^\top$
  \end{algorithm}
\end{minipage}%
\begin{minipage}[t]{4cm}
 \begin{algorithm}[H]
    \caption{$l_\infty$ Update}
    \label{alg:lU}
    \textbf{Input} Vector $\mathbf{x}, \nabla_\mathbf{x}$ \\
    \vspace*{0.56mm}
    $\Delta\mathbf{x} = \|\nabla_\mathbf{x}\|_1 \text{sign}(\nabla_\mathbf{x})$ \\
    $\mathbf{x} \leftarrow \mathbf{x} - \epsilon \Delta\mathbf{x}$
  \end{algorithm}
\end{minipage}
\vspace*{1mm}  

\begin{thm}\label{thm:SSD}
Optimizing the local approximated loss function
$\arg\min_{\bm\theta_i}\langle \nabla_{\bm\theta_i}\mathcal{L}(\bm\theta_i^{old}), \bm\theta_i - \bm\theta_i^{old} \rangle + c_i\|\bm\theta_i - \bm\theta_i^{old}\|_\infty^2$ will obtain the updates for parameters in Algorithm \ref{alg:SU}, \ref{alg:lU}.
\end{thm}

\subsubsection{Computation complexity} Another important issue w.r.t. SSD is its efficiency on the singular value decomposition (SVD). 
In most experiments, $C^{-1}$ is usually set as diagonal \cite{murphy2012machine,cho2013gaussian}, thus SVD operation is not necessary. 
For $N_v \times N_h$ weight matrix $W$, SVD requires $\mathcal{O}(N_vN_h\min\{N_v,N_h\})$ algorithmic computation. 
However, this computation cost may be practically equivalent to SGD if $\min\{N_v,N_h\}\approx Bk$, since the gradient computation by CD-$k$ algorithm requires $\mathcal{O}(N_vN_hBk)$ computations, where $B$ is the batch size. 
Alternatively, we can use the randomized SVD algorithm \cite{halko2011finding} for speeding up with low rank approximation, reducing the computational complexity to $\mathcal{O}(NM+N_vN_h\log M)$ where $N=\max\{N_v,N_h\}$ and $M$ is approximated level. 
Meanwhile, the update rule will be modified as $\Delta X=\|\bm\lambda\|_1 UV^\top+\|\bm\lambda\|_1(X-U\bm\lambda V^\top)/\|X-U\bm\lambda V^\top\|_{S_\infty}$. 
This approach is fast in practice and will not significantly harm the performance.

\subsubsection{Discussion} Additionally, if $\mathcal{L}$ is Lipschitz continuous, the convergence rate is $\mathcal{O}\left(L_p\|\bm\theta-\bm\theta^*\|_p^2 / T\right)$, where $\bm\theta^*$ is local optimal and $L_p$ is Lipschitz smoothness constant w.r.t. $p$ norm. 
This implies that SSD does not substantially increase the computational cost but may potentially encourage faster convergence than SGD (correspond to $L_2\|\bm\theta-\bm\theta^*\|_2^2$) if $L_\infty\|\bm\theta-\bm\theta^*\|_\infty^2$ is smaller. 
\cite{kelner2014almost} proved the conditions that optimization method favors infinity norm, however, it is difficult to verify for RBM with such a complicated non-convex loss function. 
This is beyond our research and we will leave it as future work.

Unlike the trend of SGD based algorithms, by applying Adagrad \cite{duchi2011adaptive}, Rmsprop \cite{tieleman2012lecture}, AdaDelta \cite{zeiler2012adadelta}, Adam \cite{kingma2014adam} on top of SGD, which are universally suitable for most deep learning problems, SSD is specifically designed for RBM. 
Additionally, SSD does not focus on how to adapt the learning rate or step-size by leveraging the gradients of previous iterations, but is basically trying to adjust the descent direction by taking advantage of geometry information of current gradient. 
However, for Gaussian RBM we notice that the covariance matrix $C^{-1/2}$ always plays the role of preconditioner in gradients derivation. 
This means we can naturally adjust or replace $C^{-1/2}$ with other preconditioner; more efficiently, we can directly apply most of the preconditioners, like Adagrad, Rmsprop, and Adam, on top of SSD. 
In summary, SSD is \textbf{parallel} to SGD.

\begin{figure*}[t]
\subfigure[BRBM Train]{
\includegraphics[width=60mm]{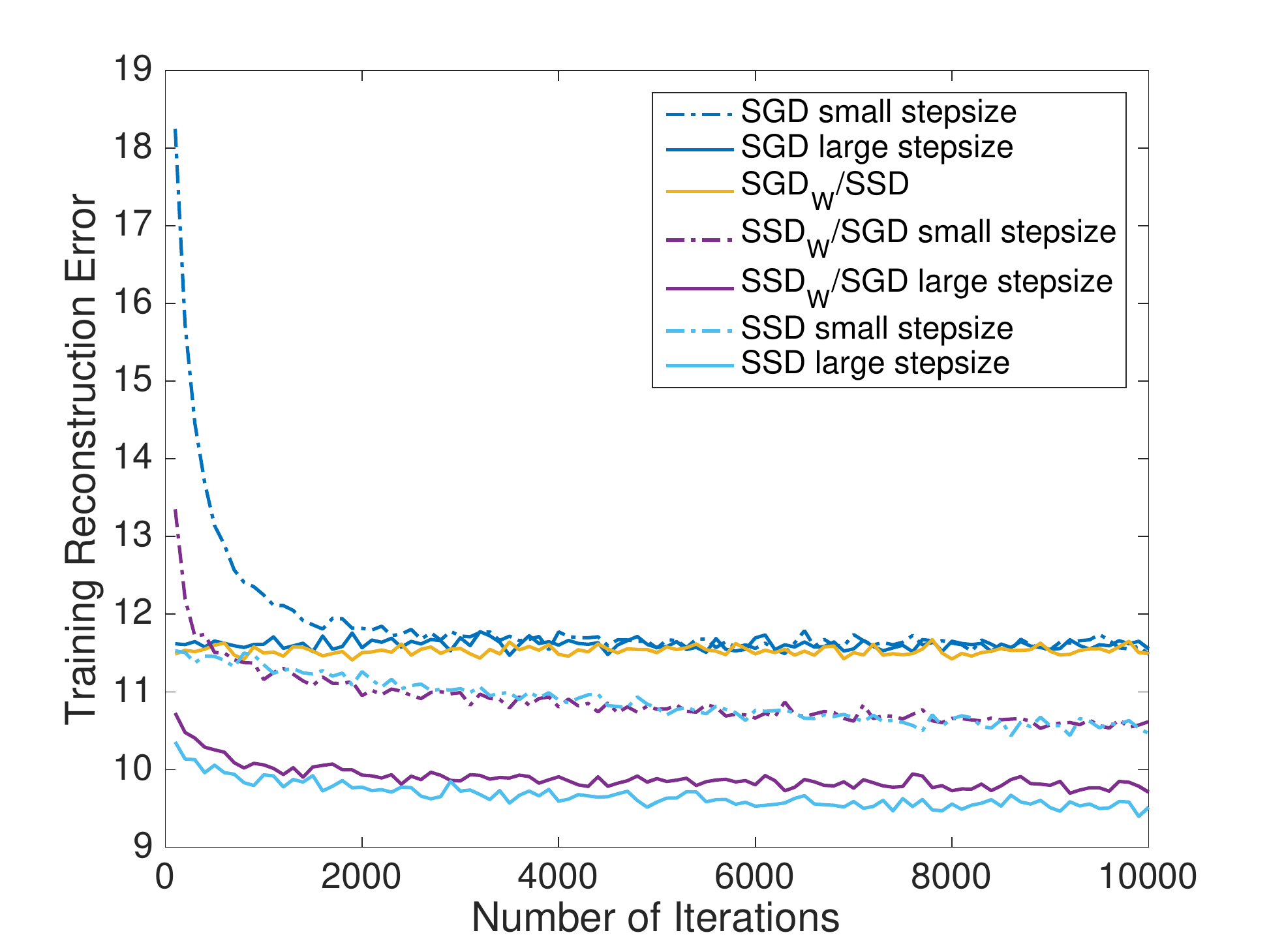}
}
\subfigure[BRBM Test]{
\includegraphics[width=60mm]{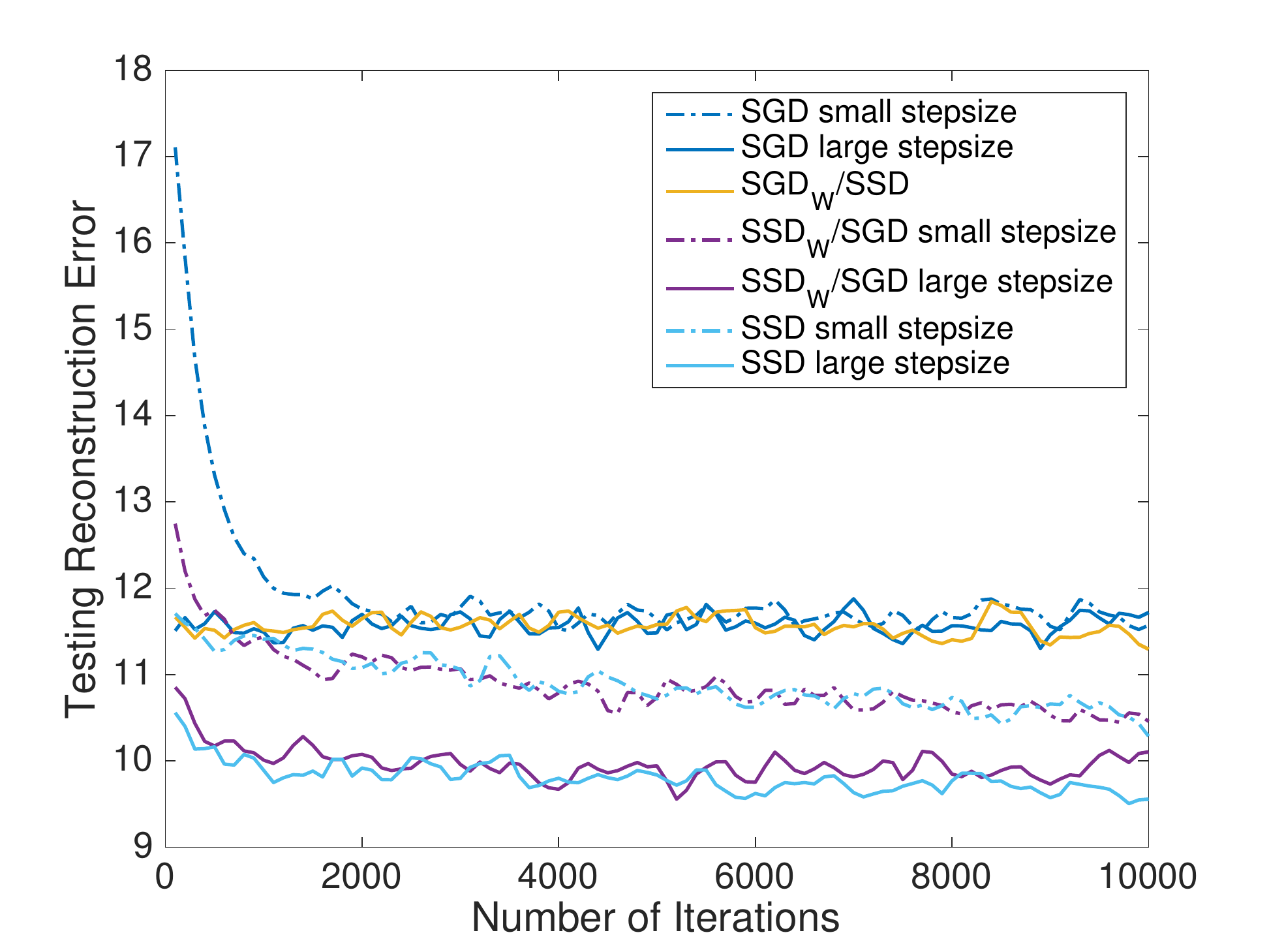}
}
\subfigure[GRBM]{
\includegraphics[width=60mm]{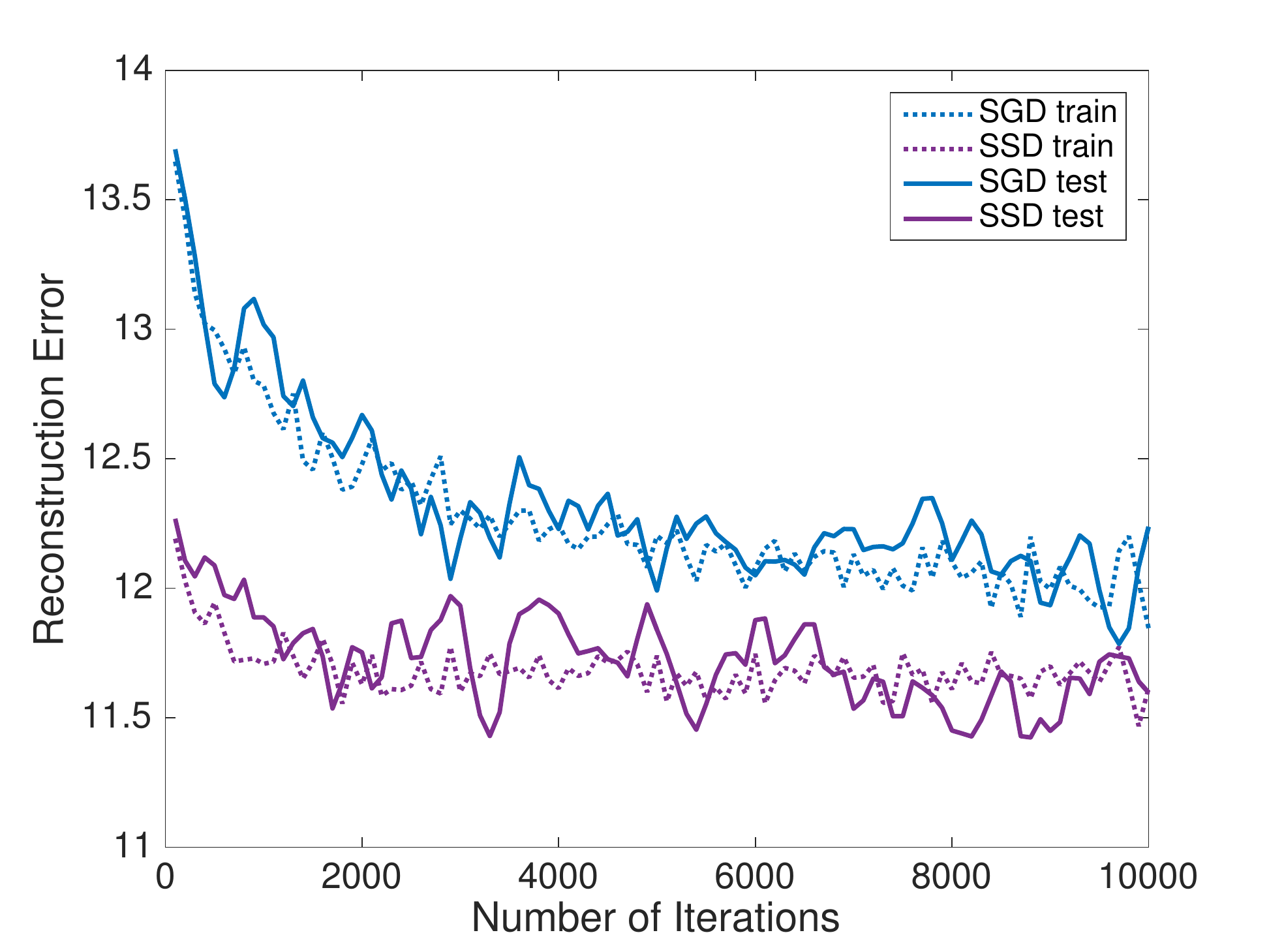}
}
\subfigure[MNIST]{
\includegraphics[width=60mm]{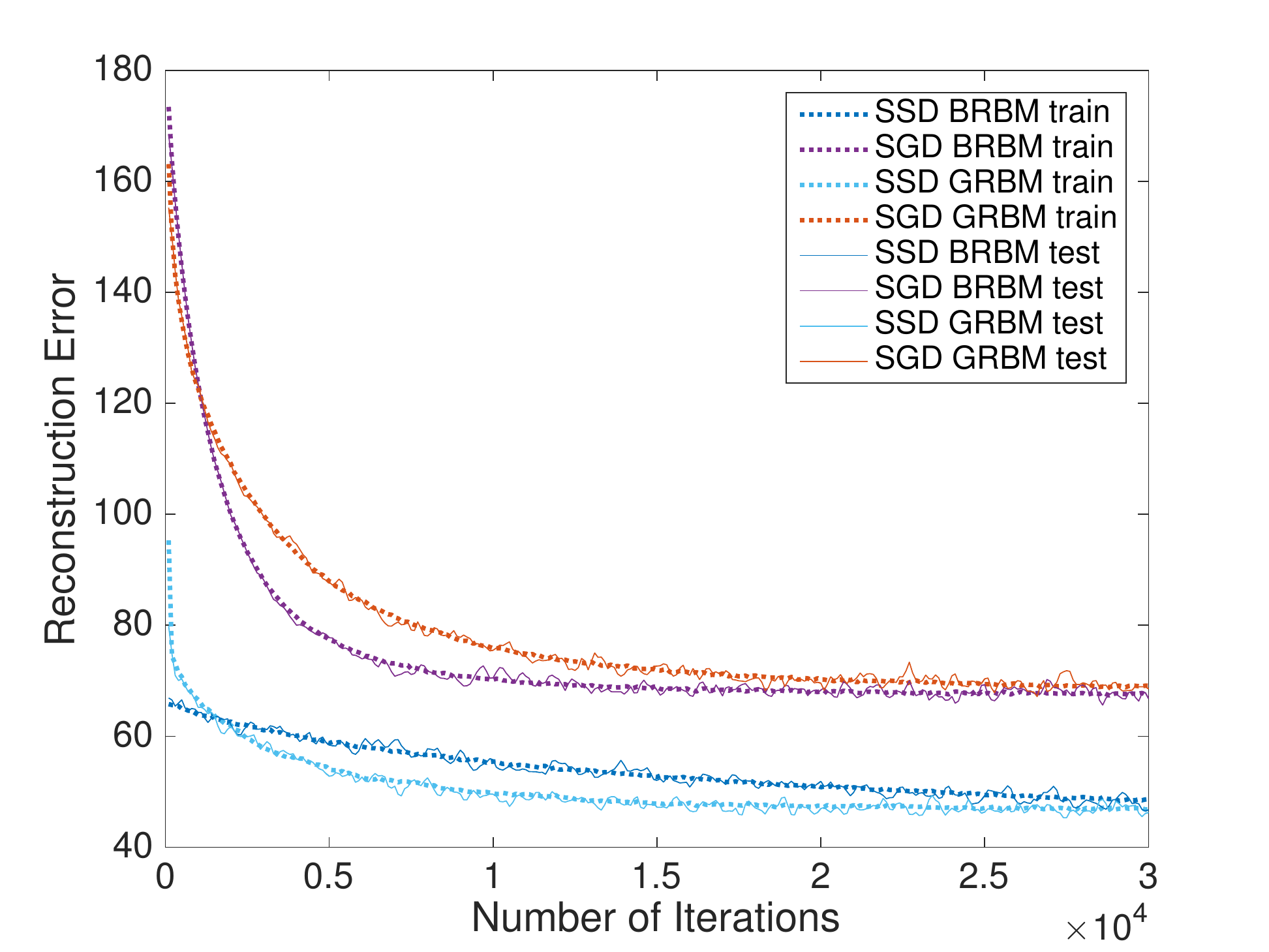}
}
\subfigure[Weights: SGD v.s. SSD]{
\includegraphics[width=60mm]{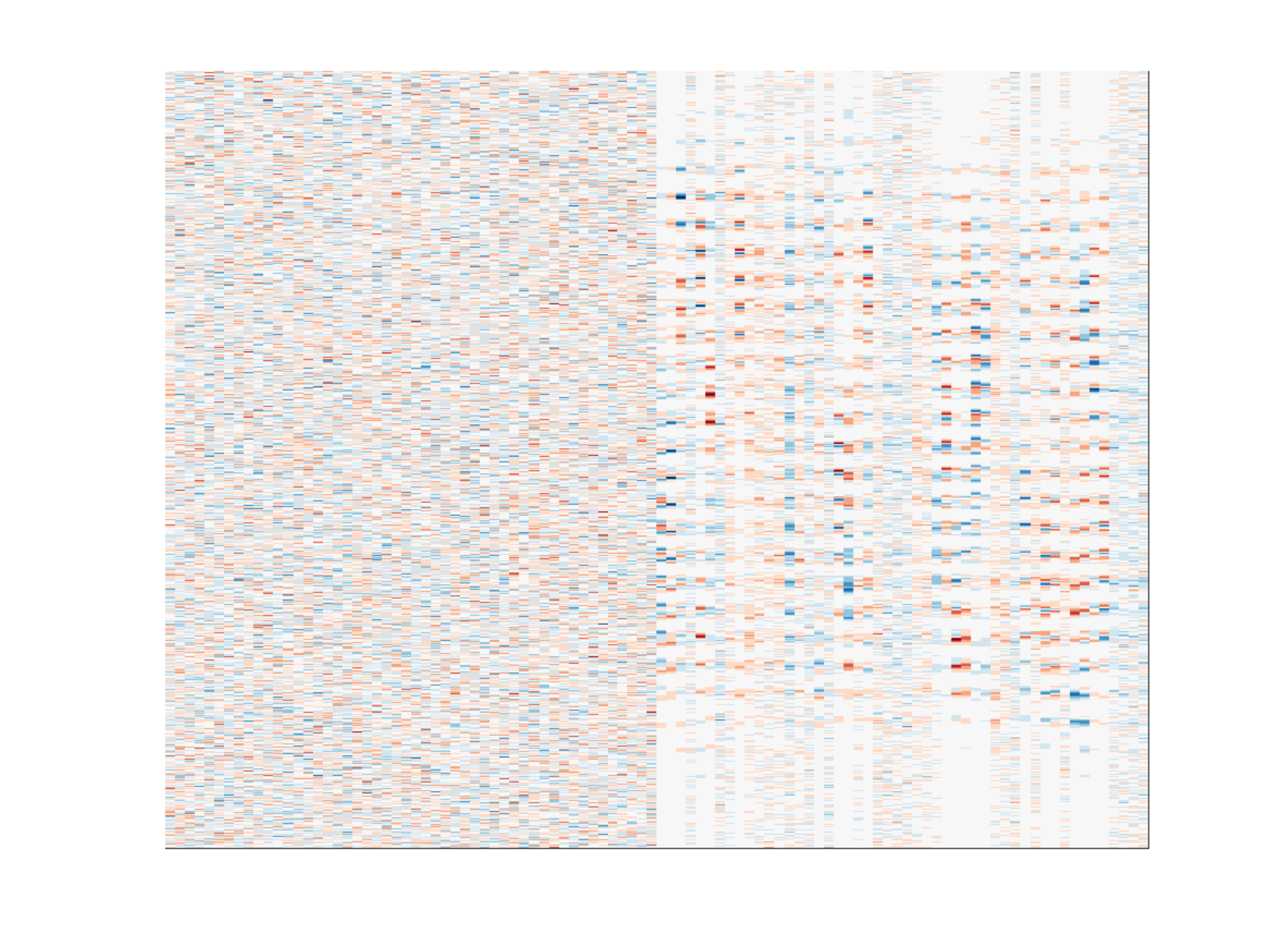}
}
\subfigure[Diagonal Covariance Matrix]{
\includegraphics[width=60mm]{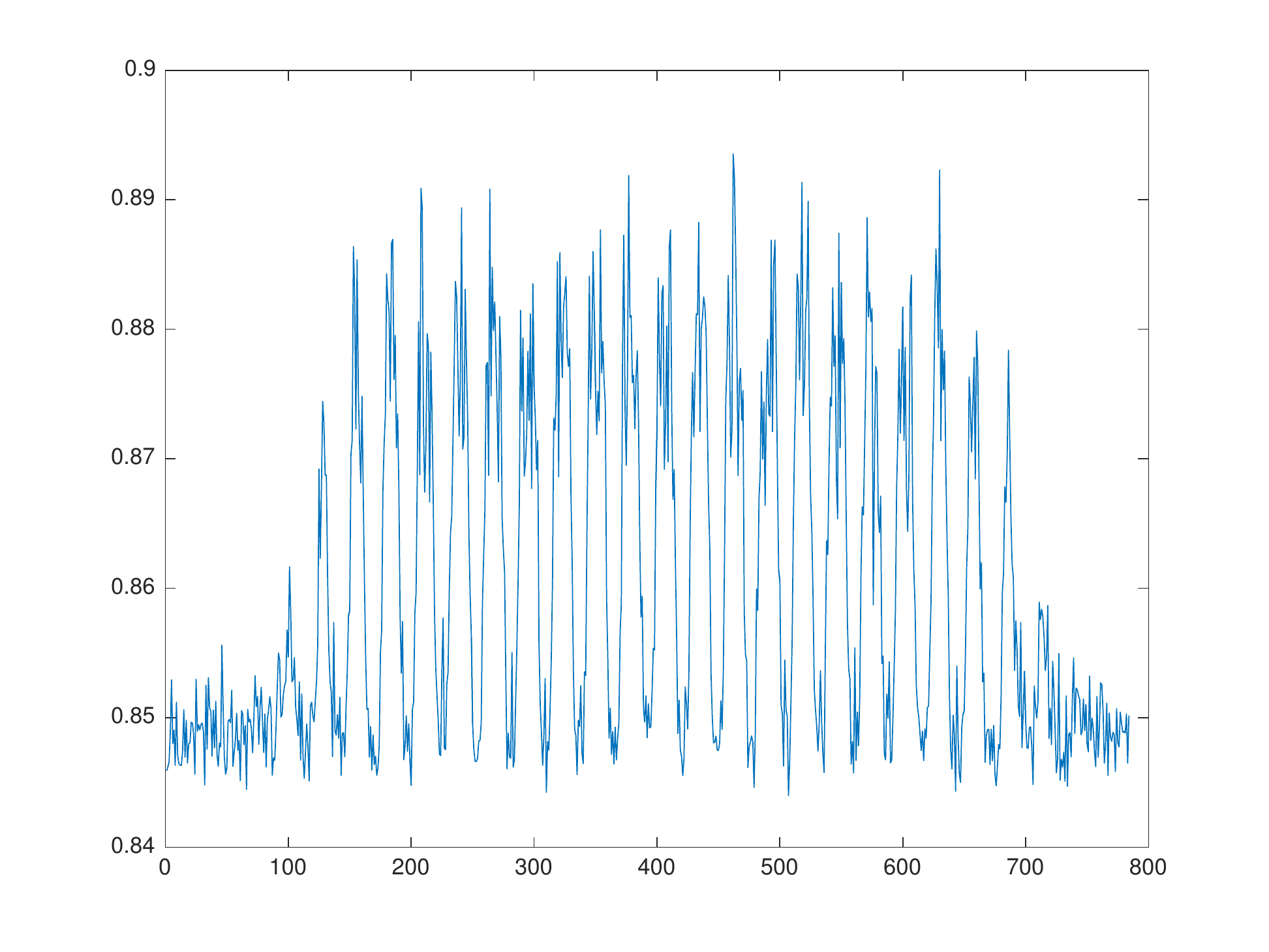}
}
\caption{\textbf{Simulation}: (a-b) SGD$_W$/SSD means $W$ is updated by SGD while other parameters are updated by SSD. The small/large step-size indicates the initials for descent algorithms; (c) For GRBM, we only show the training/testing result tuned on corresponding algorithms, while other hybrid algorithm is not considered. \textbf{MNIST}: (d) $N_h=50$ in the experiments; (e) The left is weight matrix learned by SGD and right one by SSD; (f) Diagonal covariance matrix of GRBM.}
\label{fig:simu}
\end{figure*}

\section{Experiments}
\label{sec:exp}

Since SSD is an alternative to SGD for RBM, we mainly compare the performance in vanilla version. 
As mentioned in discussion, SSD with adaptive step-size can also be implemented readily. 
However, this may misinterpret the significance of spectral descent, since adaptive step-size methods can also accelerate the practical convergence in general. 

\subsection{Simulation Study}

To quantitatively observe the behaviors or properties of SSD algorithm, our first experiment was run on synthetic dataset, with $N_v=100$ and $N_h=25$. 
4000 training and 1000 held-out testing data points are generated as binary vectors to test both Binary RBM and Gaussian RBM. 
The weights matrix $W$ was a random Gaussian matrix with a variance 0.5 for each element, and the biases are set as zero for convenience. 
However, we pretended this is an unknown information and allowed the algorithm to learn biases. 
In our model setting, we only set the $N_v$ and $N_h$ as simulated ones.

\begin{wrapfigure}{L}{5cm}
\vspace{0mm}
\centering
\hspace{2cm}
\includegraphics[width=60mm,clip,trim=0 0 0 30mm]{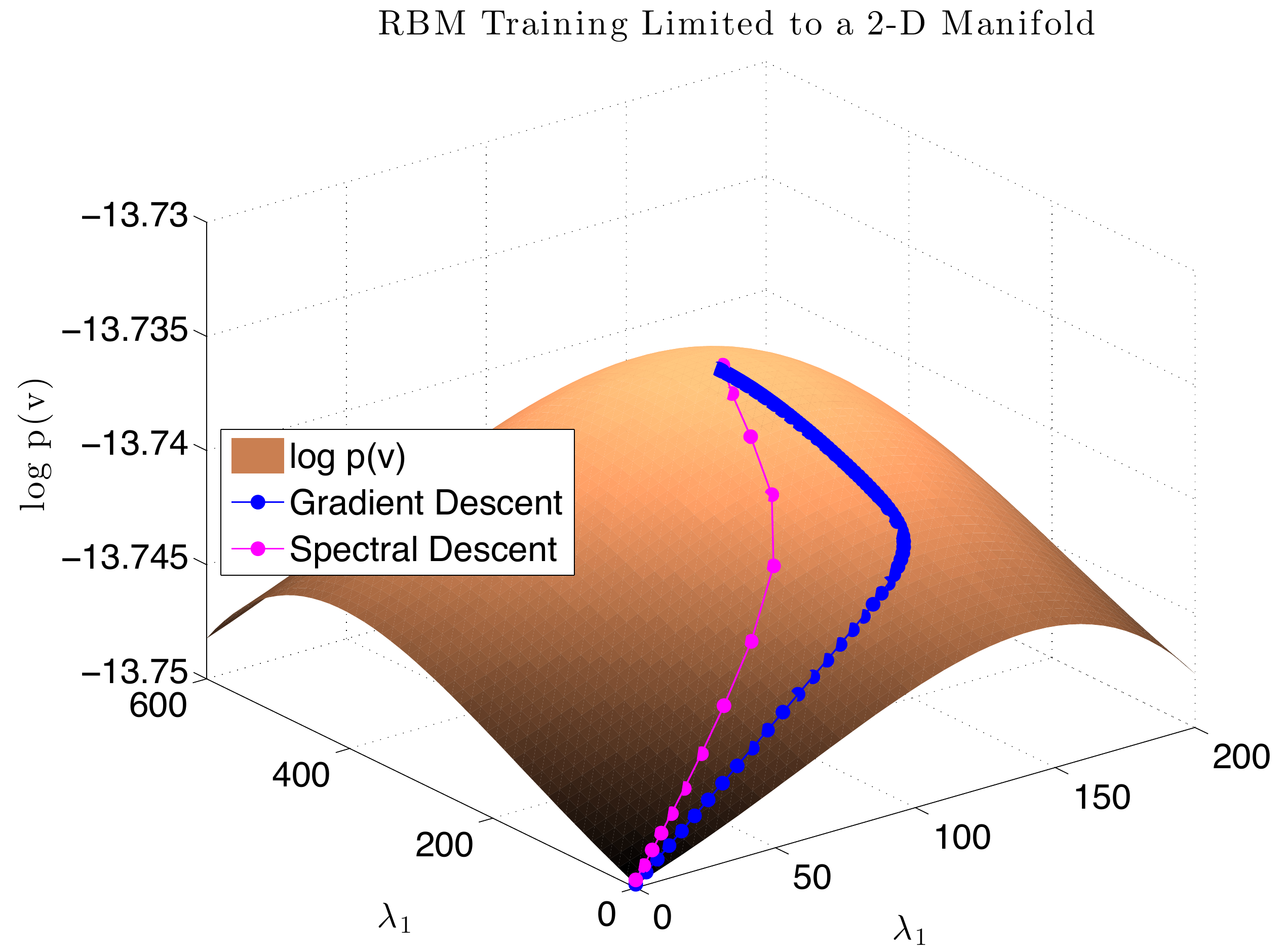}
\vspace{-2mm}
\caption{\textbf{Simulation}: GRBM training limited to a 3D manifold; in order to visualize this result is not fine-tuned.}
\label{fig:manifold}
\vspace{-2mm}
\end{wrapfigure}

The reconstruction error results are shown in Figure~\ref{fig:simu}. 
Figure~\ref{fig:simu}(a) illustrates the training performance evaluated by Binary RBM. 
To closely see the significance of SSD, we also implemented the \textbf{hybrid} descent algorithms, e.g., SSD$_W$/SGD represents only the weight matrix $W$ is updated according to spectral descent while all the other parameters follow the SGD update.  
We gradually modified updating all parameters by SGD to become updating by SSD. 
From Figure~\ref{fig:simu}(a), we can see that updating $W$ is crucial for training. 
Updating $W$ alone by SGD (i.e., SGD$_W$/SSD) has the similar behavior as updating all parameters by SGD, however, once we changed the updating mode to SSD$_W$/SGD, the behavior will become more similar to pure SSD, no matter for small step-size or large step-size. 
In addition, even we use the SSD to update all parameters, the performance is just slightly better than SSD$_W$/SGD. 
For Gaussian RBM, we use a simple model with isotropic covariance matrix $c\mathbf{I}$. 
Since we use the binary input, the GRBM should model worse than BRBM, as proved in Figure~\ref{fig:simu}(b). 
However, we can still observe that SSD converges faster than SGD, and this phenomenon can also be clearly visualized in Figure~\ref{fig:manifold}, where SSD prefers to find a shortcut to the local optimal. 

\subsection{MNIST and FreyFace Datasets}

In this experiment, we use the vectorized MNIST digit dataset that includes 60,000 training and 10,000 test images of handwritten digits with size $28\times28$. 
We test both Gaussian and Binary RBM for this dataset, though it is commonly modeled as BRBM. 
For this real dataset, we assume the covariance matrix as diagonal, i.e. $\text{diag}(\mathbf{c})$, where $\mathbf{c}$ is a vector with positive elements. 
We mainly evaluate sum of squared error for pixel-wise reconstruction, since the true normalized log-likelihood is intractable for RBM, which usually requires the annealing importance sampling \cite{salakhutdinov2008quantitative,kiwaki2015effect} to obtain an approximate result. 
Such an approximation may suffer high variance if insufficient samples are used. 
Additionally, we also evaluate the running time of SSD compared with SGD, when CD-10 is applied on a mini-batch with size 100. 
Since we found vanilla SGD converges too slow (even with non-fixed but exponentially decay stepsize), we applied the Nesterov acceleration \cite{sutskever2013importance} for SGD, but did not applied it to SSD, which seems an unfair comparison. 

\begin{wrapfigure}{L}{5cm}
\vspace{-5mm}
\centering
\hspace{-5mm}\includegraphics[height=45mm]{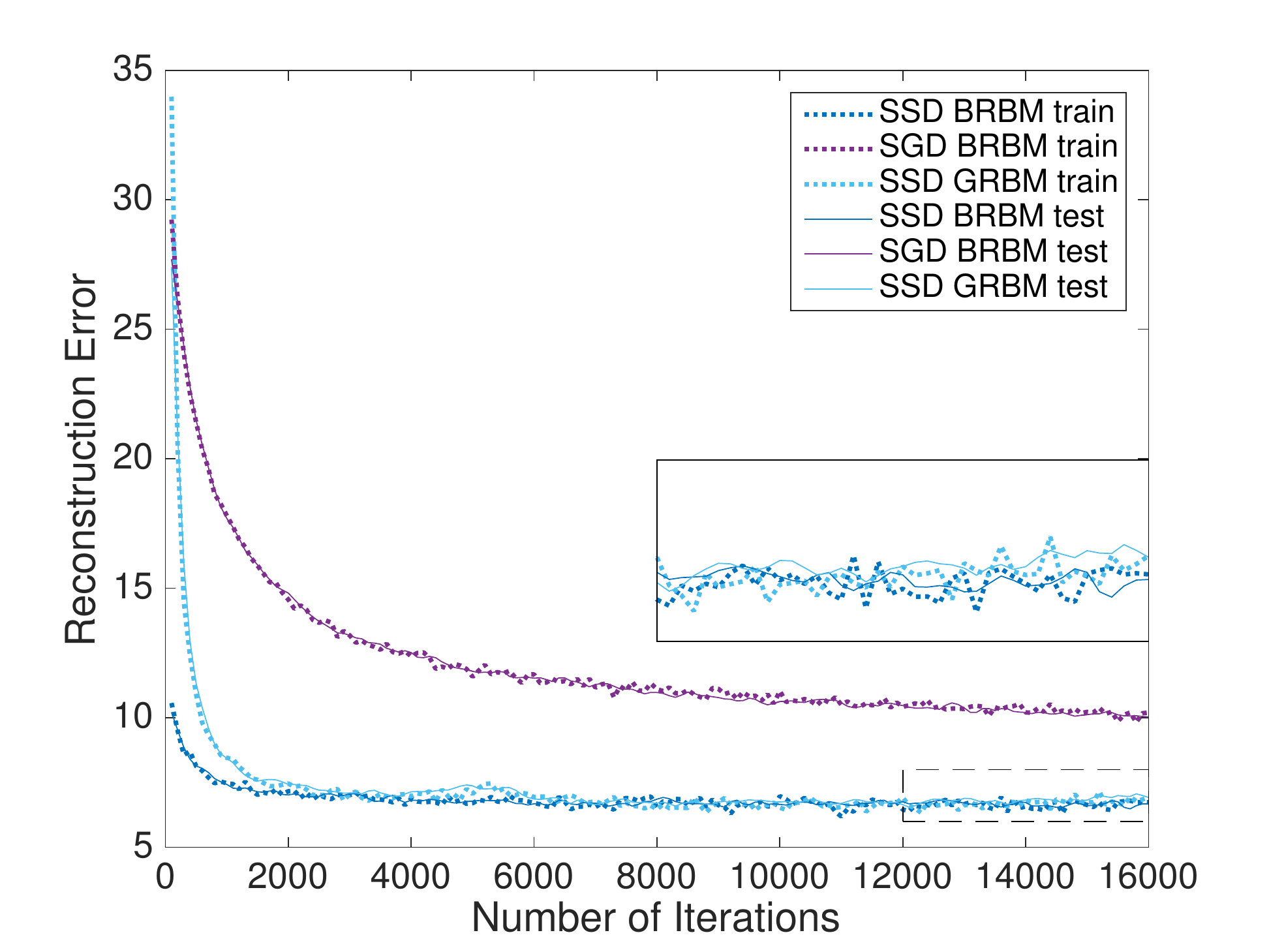}
\vspace{-4mm}
\caption{\textbf{FreyFace}: $N_h=200$.}
\label{fig:ff}
\vspace{-4mm}
\end{wrapfigure}

Figure~\ref{fig:simu}(d) shows the results running on MNIST. 
We found SSD on GRBM can achieve faster convergence rate than BRBM, though the error of initial hundreds of iterations is relatively large. 
Unlike the simulation, this is a counter-intuitive phenomenon, since our images has been binarized. 
However, one possible reason is that the covariance matrix enables more flexible modeling, allowing the more variation for same pixel position. 
Even we use the Nesterov SGD, there still exists a gap between the SSD and SGD for the first 30,000 iterations, where the performance of SGD has become stable but SSD implicates a trend for further decline. 
Figure~\ref{fig:simu}(e) shows the different weight matrix learned by SGD and SSD. 
The SGD learned matrix is more random while the SSD can learn more sparsity or patterns. 
If we plot the histogram for all the elements of $W$, we can also observe the similar sparsity for SSD (See figures in supplements).

We also compared the running time between SGD and SSD. 
For fair comparison, we also include the vanilla version of SGD. 
The results of wall-clock time illustrated in Table~\ref{tab:mnist} make sense since SVD and momentum update will require extra computations. 

The FreyFace\footnote{\scriptsize{\url{http://cs.nyu.edu/~roweis/data/frey_rawface.mat}}} dataset contains almost 2000 images of some person's face, taken from sequential frames of a small video with size $20\times28$. 
Each pixel value of face images in this dataset is real value between 0 and 1, which should fit GRBM well. 
We can also consider it as the Bernoulli probability for observation and apply BRBM, thus we test both BRBM and GRBM and show the result in Figure~\ref{fig:ff}. 
Similar to MNIST experiment, Nesterov acceleration is implemented for SGD, however, SGD on GRBM converges too slow even with thoroughly fine-tuning. 
Thus, we did not include this result in the figure. 
A counter-intuitive phenomenon also appears for SSD, i.e., BRBM can model the data better than GRBM, even if the data point is real-valued. 
One possible explanation is that the size of this dataset is small such that the more flexible GRBM is not well trained. 
 
\begin{table}[t]
\setlength\tabcolsep{4pt}
\centering
\caption{Running time per 1k iterations}
\begin{tabular}{|c|c|c|c|} \hline
 & \textbf{SSD} & \textbf{SGD} & \textbf{Nesterov SGD} \\ \hline
\textbf{GRBM} & 67s & 54s & 58s\\
\textbf{BRBM} & 48s & 44s & 45.5s\\
\hline\end{tabular}
\label{tab:mnist}
\end{table}

\section{Conlusion and Future Work}
\label{sec:conl}

In this paper, we generalize the stochastic spectral descent algorithm to a broader RBM family, enlarging the domain of SSD. 
A unifying SSD algorithm is summarized due to the newly proved upper bound. 
Since SSD empirically shows potential advantages over stochastic gradient descent, the emergence of SSD may imply an alternative or substitute of SGD for RBM related model, or adaptive step-size modification.
An important avenue of improvement might explore other gradient estimation method for RBM and integrate with SSD. 
For example, Hamiltonian Monte Carlo (HMC) method \cite{neal2011mcmc} allows efficient sampling from conditional distributions that might be difficult to sample from, like mean-covariance RBM \cite{hinton2010modeling,krizhevsky2010factored}. 
HMC may also be helpful to improve the overall performance of SSD, and the interaction between HMC and SSD is our future work. 
Additionally, beyond the discussion of traditional $l_2$ norm and infinite norm,  $l_0$ or $l_1$ norm constrain can be a new direction to derive other optimization algorithms which allow partial updates for its sparse property. 

\bibliography{refs.bib}
\bibliographystyle{plain}

\clearpage

\section{Appendices}

\subsection{Gradient Computation}
\begin{align*}
\frac{\partial \mathcal{L}}{\partial W} &= \mathbb{E}_{p}[C^{-1}\mathbf{vh}^{\top}] - \mathbb{E}_{e}[C^{-1}\mathbf{vh}^{\top}] \\
\frac{\partial \mathcal{L}}{\partial C^{-1}} &= \mathbb{E}_{p}[W\mathbf{h}\mathbf{v}^\top-(\mathbf{v}-\mathbf{b})(\mathbf{v}-\mathbf{b})^{\top}/2]- \mathbb{E}_{e}[W\mathbf{h}\mathbf{v}^\top-(\mathbf{v}-\mathbf{b})(\mathbf{v}-\mathbf{b})^{\top}/2] \\
\frac{\partial \mathcal{L}}{\partial \mathbf{b}} &= \mathbb{E}_{p}[C^{-1}\mathbf{v}] - \mathbb{E}_{e}[C^{-1}\mathbf{v}] \\
\frac{\partial \mathcal{L}}{\partial \mathbf{a}} &= \mathbb{E}_{p}[\mathbf{h}] - \mathbb{E}_{e}[\mathbf{h}] 
\end{align*}
Consider special case, $C^{-1}$ is diagonal, thus denoting as a vector.
\begin{align*}
\frac{\partial \mathcal{L}}{\partial C^{-1}} &= \mathbb{E}_{p}[\mathbf{v}\odot W\mathbf{h} -(\mathbf{v}-\mathbf{b})\odot(\mathbf{v}-\mathbf{b})/2]- \mathbb{E}_{e}[\mathbf{v}\odot W\mathbf{b} -(\mathbf{v}-\mathbf{b})\odot(\mathbf{v}-\mathbf{b})/2] \\
\frac{\partial \mathcal{L}}{\partial \mathbf{b}} &= \mathbb{E}_{p}[C^{-1}\odot\mathbf{v}] - \mathbb{E}_{e}[C^{-1}\odot\mathbf{v}] 
\end{align*}
Furthermore, the element of diagonal should be positive, $C=\text{diag}\{e^{c_1},\dots,e^{c_R}\}$.
\begin{align*}
\frac{\partial \mathcal{L}}{\partial \ln (C)} &= C^{-1}\odot \left\{ \mathbb{E}_{p}[(\mathbf{v}-\mathbf{b})\odot(\mathbf{v}-\mathbf{b})/2 - \mathbf{v}\odot W\mathbf{h}]- \mathbb{E}_{e}[(\mathbf{v}-\mathbf{b})\odot(\mathbf{v}-\mathbf{b})/2- \mathbf{v}\odot W\mathbf{b} ] \right\}
\end{align*}
if $C^{-1}=\text{diag}\{e^{c_1},\dots,e^{c_R}\}$,
\begin{align*}
\frac{\partial \mathcal{L}}{\partial \ln (C^{-1})} &= C^{-1}\odot\left\{ \mathbb{E}_{p}[\mathbf{v}\odot W\mathbf{h}-(\mathbf{v}-\mathbf{b})\odot(\mathbf{v}-\mathbf{b})/2]- \mathbb{E}_{e}[\mathbf{v}\odot W\mathbf{b}-(\mathbf{v}-\mathbf{b})\odot(\mathbf{v}-\mathbf{b})/2] \right\}
\end{align*}

\subsection{Derivation of $g(\bm\theta)$}
$N$ is the size of dataset.
\begin{align*}
g(\bm\theta)&=-\frac{1}{N}\sum_{n=1}^N\log\left(\sum_{\mathbf{h}}\exp\left( \mathbf{v}_n^\top C^{-1} W\mathbf{h}-\frac{1}{2}(\mathbf{v}_n-\mathbf{b})^\top C^{-1}(\mathbf{v}_n-\mathbf{b})+\mathbf{h}^\top\mathbf{a} \right)\right) \\
&=-\frac{1}{N}\sum_{n=1}^N\log\left(\exp\left(-\frac{1}{2}(\mathbf{v}_n-\mathbf{b})^\top C^{-1}(\mathbf{v}_n-\mathbf{b})\right)\prod_{k=1}^{N_h}\left(1+\exp\left(\mathbf{v}_n^\top C^{-1}W_{\cdot,k}+ a_k\right)\right)\right) \\
&=\frac{1}{N}\sum_{n=1}^N \left(\frac{1}{2}(\mathbf{v}_n-\mathbf{b})^\top C^{-1}(\mathbf{v}_n-\mathbf{b})-\sum_{k=1}^{N_h}\log\left( 1+\exp\left(\mathbf{v}_n^\top C^{-1} W_{\cdot,k}+ a_k\right)\right)\right)
\end{align*}
The first part of $g$ is linear to $C^{-1}$, and the second part is concave w.r.t. $W$, $\mathbf{a}$, and $C^{-1}$, since $-\log(1+e^x)$ is concave.
\begin{align*}
g(\mathbf{b}+\Delta\mathbf{b}) &= g(\mathbf{b}) + \langle g'(\mathbf{b}), \Delta\mathbf{b} \rangle + \frac{1}{2}\Delta\mathbf{b}^{\top} C^{-1} \Delta\mathbf{b} 
\Rightarrow g(\mathbf{b}+\Delta\mathbf{b}) \leq g(\mathbf{b}) + \langle g'(\mathbf{b}), \Delta\mathbf{b} \rangle + \frac{N_v r(C)}{2} \|\Delta\mathbf{b}\|_\infty^2
\end{align*}
One option for $r(C)$ is $1/\lambda_{\min}(C)$, where $\lambda_{\min}(C)$ is minimum eigenvalue of $C$.

\subsection{Theorem Proof}
\begin{lem}\label{lem:lse} 
Consider log-sum-exp function $lse_{\bm\omega}(\mathbf{x})=\log\sum_{i}\omega_i\exp(x_i)$ where $\bm\omega$ is independent of $\mathbf{x}$, the following inequality $lse_{\bm\omega}(\mathbf{x}+\Delta\mathbf{x}) \leq lse_{\bm\omega}(\mathbf{x}) + \langle \nabla lse_{\bm\omega}(\mathbf{x}), \Delta\mathbf{x}\rangle + \frac{1}{2}\|\Delta\mathbf{x}\|_\infty $ holds. 
\end{lem}
\begin{proof}
See \cite{carlson2015stochastic}.
\end{proof}
\begin{thm}\label{thm:bound_a} When fixing other parameters, we have
\begin{align}
f(\mathbf{a}+\Delta\mathbf{a}) &\leq f(\mathbf{a}) + \langle \nabla_{\mathbf{a}} f(\mathbf{a}), \Delta\mathbf{a}\rangle + \frac{N_h}{2}\|\Delta\mathbf{a}\|_\infty^2  \label{eq:bound_a}
\end{align}
\end{thm}

\begin{proof}
The log partition function can be written as a sum over only the hidden units to give a similar form to Lemma \ref{lem:lse}.
\begin{align*}
f(\bm\theta) &=\log\left(\int\sum_{\mathbf{h}}\exp\left(\mathbf{v}^\top C^{-1} W\mathbf{h}-\frac{1}{2}(\mathbf{v}-\mathbf{b})^\top C^{-1}(\mathbf{v}-\mathbf{b})+\mathbf{h}^\top\mathbf{a}\right)\mathrm{d}\mathbf{v}\right)\\
&= \log\left(\sum_{i=1}^{2^{N_h}}\omega_i\exp(\mathbf{h}_i^\top\mathbf{a}) \right)
\end{align*}
where 
\begin{align*}
\omega_i &=\int \exp\left(\mathbf{v}^\top C^{-1} W\mathbf{h}_i-\frac{1}{2}(\mathbf{v}-\mathbf{b})^\top C^{-1}(\mathbf{v}-\mathbf{b})\right)\mathrm{d}\mathbf{v} \\
&= (2\pi)^{\frac{N_v}{2}}|C|^{\frac{1}{2}}\exp\left(\mathbf{b}^\top C^{-1} W\mathbf{h}_i+ \frac{(W\mathbf{h}_i)^\top C^{-1} W\mathbf{h}_i}{2} \right)
\end{align*} 
and it does not depend on $\mathbf{a}$. 
Define $\mathbf{H}\in\{0,1\}^{N_h,2^{N_h}}$ be the matrix with each column taking one possible configuration of $\mathbf{h}$, thus we have
\begin{align*}
f(\bm\theta) = \log\left(\bm\omega^\top\exp(\mathbf{H}^\top\mathbf{a}) \right)
\end{align*}
By Lemma \ref{lem:lse}, we have
\begin{align*}
f(\mathbf{a} + \Delta\mathbf{a}) \leq f(\mathbf{a}) + \langle \nabla_{\mathbf{H}^\top\mathbf{a}} lse(\mathbf{H}^\top\mathbf{a}), \mathbf{H}^\top\Delta\mathbf{a}\rangle + \frac{1}{2}\|\mathbf{H}^\top\Delta\mathbf{a}\|_\infty^2
\end{align*}
where we omit other fixed parameters in $\bm\theta$. Notice the fact that $\mathbf{H} \nabla_{\mathbf{H}^\top\mathbf{a}} lse(\mathbf{H}^\top\mathbf{a})= \nabla_{\mathbf{a}}f(\mathbf{a})$, thus the $(\nabla_{\mathbf{H}^\top\mathbf{a}} lse(\mathbf{H}^\top\mathbf{a}))^\top \mathbf{H}^\top\Delta\mathbf{a} = (\nabla_{\mathbf{a}}f(\mathbf{a}))^\top \Delta\mathbf{a}$ holds.

In addition, $\|\mathbf{H}^\top\Delta\mathbf{a}\|_\infty^2=\max_i|\mathbf{h}_i^\top \Delta\mathbf{a}|\leq N_h\|\Delta\mathbf{a}\|_\infty^2$.
\end{proof}

\begin{thm}\label{thm:bound_b} When fixing other parameters, we have
\begin{align}
f(\mathbf{b}+\Delta\mathbf{b}) &\leq f(\mathbf{b}) + \langle \nabla_{\mathbf{b}} f(\mathbf{b}), \Delta\mathbf{b}\rangle + \frac{N_hRr(C)}{2}\|\Delta\mathbf{b}\|_\infty^2 \label{eq:bound_b}
\end{align}
where $r(C)$ is a local constant dependent on the parameter in previous iteration.
\end{thm}
\begin{proof}
We can also rewrite log partition function as the following formulation.
\begin{align*}
f(\bm\theta) = \log\left(\sum_{i=1}^{2^{N_h}}\omega_i\exp(\mathbf{b}^\top C^{-1}W\mathbf{h}_i) \right)
\end{align*}
where $\omega_i=(2\pi)^{\frac{N_v}{2}}|C|^{\frac{1}{2}}\exp\left(\mathbf{h}_i^\top\mathbf{a}+ \frac{1}{2}(W\mathbf{h}_i)^\top C^{-1}W\mathbf{h}_i \right)$. 
Following the derivation of proof in Theorem \ref{thm:bound_a}, we only need to consider $\max_i|\Delta\mathbf{b}^\top C^{-1}W\mathbf{h}_i|$. 
Using the assumption $\|W\|_2 \leq R$, this term is bounded by $N_hR\|\Delta\mathbf{b}^\top C^{-1}\|_\infty^2$. 
Denote $\lambda_{\min}(C)$ and $\lambda_{\max}(C)$ as the maximum and minimum eigen values of covariance matrix $C$, thus there exist a constant $r(C)\in[\lambda_{\min}(C), \lambda_{\max}(C)]$ such that the desired term bounded by $N_hRr(C)\|\Delta\mathbf{b}\|_\infty^2$. 
Obviously, $r(C)=1/\lambda_{\min(C)}$ is one option.
\end{proof}

\begin{prop}\label{prop:logdet} 
$-\log|C^{-1}|=\log|C|$.
\end{prop}
\begin{thm}\label{thm:bound_C} When fixing other parameters, we have
\begin{align}
f(C^{-1}+U) &\leq f(C^{-1}) + \langle \nabla_{C^{-1}} f(C^{-1}), U\rangle + \left(\frac{N_h^2R^2}{2}+r(C)\right)\|U\|_{S_\infty}^2 \label{eq:bound_C}
\end{align}
where $U$ is positive semidefinite and inner product is defined as trace for matrix.
\end{thm}
\begin{proof}
With Proposition \ref{prop:logdet} we can rewritten log partition function as
\begin{align*}
f(\bm\theta) =& \log\left(\sum_{i=1}^{2^{N_h}}\omega_i\exp\left( \left(\mathbf{b}+\frac{1}{2}(W\mathbf{h}_i)\right)^\top C^{-1}W\mathbf{h}_i \right) \right) - \frac{1}{2}\log|C^{-1}|
\end{align*}
where $\omega_i=(2\pi)^{\frac{N_v}{2}}\exp\left(\mathbf{h}_i^\top\mathbf{a}\right)$. We also have
\begin{align*}
-\log|C^{-1}+U| - (-\log|C^{-1}|) = - \langle C^\top, U\rangle - \int_0^1 (1-t)vec(U)^\top \nabla^2 \log|C^{-1}+tU|vec(U)\mathrm{d}t = - \langle C^\top, U\rangle + A
\end{align*}
wher $A$ is non-negative and there exists a constant $r'(C)$ such that $A\leq r'(C)\|vec(U)\|_2^2=r'(C)\|U\|_F^2=r'(C)\|\bm\lambda(U)\|_2^2$ where $\bm\lambda(U)$ is the singular value vector, thus there exists a constant $r(C)=N_v^2r'(C)$ such that $A \leq r(C)\|U\|_{S_\infty}^2$.
\end{proof}

\begin{lem}\label{lem:lse2} 
Consider log-sum-exp-square function $lse2_{\bm\omega}(\mathbf{x})=\log\sum_{i}\omega_i\exp(x_i^2/2)$ where $\bm\omega$ is independent of $\mathbf{x}$, if the domain of this function satisfies the bound condition $\|\mathbf{x}\|_2\leq r$, then the following inequality $lse2_{\bm\omega}(\mathbf{x}+\Delta\mathbf{x}) \leq lse2_{\bm\omega}(\mathbf{x}) + \langle \nabla lse2_{\bm\omega}(\mathbf{x}), \Delta\mathbf{x}\rangle + \left(\frac{1}{2}+\frac{3r^2}{4} \right)\|\Delta\mathbf{x}\|_\infty $ holds. 
\end{lem}
\begin{proof}
Denote $lse2_{\bm\omega}(\mathbf{x})$ as $l(\mathbf{x})$ for simplicity.
\begin{align*}
\frac{\partial}{\partial x_i} l(\mathbf{x})&= \frac{x_i\omega_i\exp(x_i^2/2)}{s(\mathbf{x})} \\
\frac{\partial^2}{\partial x_i\partial x_j} l2(\mathbf{x})&= \delta_{ij}(1+x_i^2)\frac{\omega_i\exp(x_i^2/2)}{s(\mathbf{x})} -x_ix_j\frac{\omega_i\exp(x_i^2/2)\omega_j\exp(x_j^2/2)}{s(\mathbf{x})^2} 
\end{align*}
where $s(\mathbf{x}) = \exp(l(\mathbf{x}))$. Let $\pi_i=\frac{\omega_i\exp(x_i^2/2)}{s(\mathbf{x})}$, then $\sum_i \pi_i=1$, $(1+x_i^2)\pi_i - x_i^2\pi_i^2 \leq \pi_i + \frac{x_i^2}{4}$.
Then for $\mathbf{u}$ where $|u_i|<|x_i|,\forall i$,
\begin{align*}
l(\mathbf{x}+\mathbf{u}) = l(\mathbf{x}) + \langle\nabla l(\mathbf{x}),\mathbf{u} \rangle + \int_0^1 (1-t)\mathbf{u}^\top \mathbf{H}_{l(\mathbf{x}+t\mathbf{u})} \mathbf{u}\mathrm{d}t
\end{align*}
where $\mathbf{H}_{l(\mathbf{x}+t\mathbf{u})}$ is the hessian matrix of $f$ evaluated at $\mathbf{x}+t\mathbf{u}$.
The integral remainder can be locally bounded, since the following statement.
\begin{align*}
\mathbf{u}^\top \mathbf{H}_{l(\mathbf{x}+t\mathbf{u})} \mathbf{u} \leq \mathbf{u}^\top \text{diag}(\pi_i+(x_i+tu_i)^2/4) \mathbf{u} \leq \left(1+\|\mathbf{x}\|_2^2\right)\|\mathbf{u}\|_\infty^2
\end{align*}

Or we can write
\begin{align*}
\int_0^1(1-t)\mathbf{u}^\top \mathbf{H}_{f(\mathbf{x}+t\mathbf{u})}\mathrm{d}t \mathbf{u} &\leq \int_0^1(1-t)\mathbf{u}^\top \text{diag}(\pi_i+(x_i+tu_i)^2\pi_i) \mathbf{u} \mathrm{d}t \\
&\leq \int_0^1(1-t)\left(\|\mathbf{u}\|_\infty^2 + \|\mathbf{x}\|_\infty^2\|\mathbf{u}\|_\infty^2 + t^2\|\mathbf{u}\|_\infty^4 + t\|\mathbf{x}\|_\infty\|\mathbf{u}\|_\infty^3 \right) \mathrm{d}t \\
&=\frac{1}{2}\|\mathbf{u}\|_\infty^2 + \frac{1}{2}\|\mathbf{x}\|_\infty^2\|\mathbf{u}\|_\infty^2+\frac{1}{6}\|\mathbf{x}\|_\infty\|\mathbf{u}\|_\infty^3+\frac{1}{12}\|\mathbf{u}\|_\infty^4
\end{align*}
Thus we conclude the Lemma.
\end{proof}

\begin{thm}\label{thm:bound_W} When fixing other parameters, we have
\begin{align*}
f(W+U) &\leq f(W) + \langle \nabla_W f(W), U\rangle + \left(\frac{1}{2}+\frac{3r(R,C,\mathbf{b})^2}{4}\right)N_vN_h\|U\|_{S_\infty}^2 \label{eq:bound_W}
\end{align*}
where $r(R,C,\mathbf{b})$ depends on the parameter $C,\mathbf{b}$ in previous iteration.
\end{thm}
\begin{proof}
Rewrite log partition function as quadratic form.
\begin{align*}
f(\bm\theta)=\log\left(\sum_{i=1}^{2^{N_h}}\omega_i\exp\left(\frac{1}{2} s(W)^\top  s(W) \right) \right)
\end{align*}
where $\omega_i=(2\pi)^{\frac{N_v}{2}}|C|^{\frac{1}{2}}\exp\left(\mathbf{h}_i^\top\mathbf{a}-\frac{1}{2}\mathbf{b}^\top C^{-1}\mathbf{b}\right)$ and $s(W) = C^{-1/2}(W\mathbf{h}_i+\mathbf{b})$. 
Follow Lemma \ref{lem:lse2} and previous proof can conclude the result. 
Notice $r(R,C,\mathbf{b})$ also depends the eigenvalues of $C$.
\end{proof}

\begin{thm}\label{thm:SSD}
Optimizing the local approximated loss function
$\arg\min_{\bm\theta_i}\langle \nabla_{\bm\theta_i}\mathcal{L}(\bm\theta_i^{old}), \bm\theta_i - \bm\theta_i^{old} \rangle + c_i\|\bm\theta_i - \bm\theta_i^{old}\|_\infty^2$ will obtain the updates for parameters in Algorithm 2, 3.
\end{thm}
\begin{proof}
If $\bm\theta_i$ is a vector, we have
\begin{align*}
&\arg\min_{\bm\theta_i}\langle \nabla_{\bm\theta_i}\mathcal{L}(\bm\theta_i^{old}), \bm\theta_i - \bm\theta_i^{old} \rangle + c_i\|\bm\theta_i - \bm\theta_i^{old}\|_{l_\infty}^2 \\
=& \arg\min_{\bm\theta_i, t, |\theta_{ij}|<t, t>0}\langle \nabla_{\bm\theta_i}\mathcal{L}(\bm\theta_i^{old}), \bm\theta_i - \bm\theta_i^{old} \rangle + c_i t^2 \\
=& \arg\min_{\bm\theta_i, t, |\theta_{ij}|<t, t>0}\langle \nabla_{\bm\theta_i}\mathcal{L}(\bm\theta_i^{old}), -t\cdot\text{sign}(\nabla_{\bm\theta_i}\mathcal{L}(\bm\theta_i^{old})) \rangle + c_i t^2
\end{align*}
This will conclude the update in Algorithm 3.

If $\bm\theta_i$ is a matrix, we first apply SVD on $\nabla_{\bm\theta_i}\mathcal{L}(\bm\theta_i^{old})=U\Sigma V^\top$. Thus we can decompose $\Delta\bm\theta_i=\bm\theta_i-\bm\theta_i^{old}$ as $U\Lambda V^\top$.

We further have
\begin{align*}
&\arg\min_{\bm\theta_i}\langle \nabla_{\bm\theta_i}\mathcal{L}(\bm\theta_i^{old}), \Delta\bm\theta_i \rangle + c_i\|\Delta\bm\theta_i\|_{S_\infty}^2 \\
=& \arg\min_{\Delta\bm\theta_i, t, \|\Delta\bm\theta_i\|_{S_\infty}<t, t>0} \text{tr}(\nabla_{\bm\theta_i}\mathcal{L}(\bm\theta_i)\Delta\bm\theta_i) + c_i t^2 \\
=& \arg\min_{\Lambda, t, \|\Lambda\|_{S_\infty}<t, t>0} \bm\lambda\left(\nabla_{\bm\theta_i}\mathcal{L}(\bm\theta_i^{old})\right)^\top \Lambda + c_i t^2 
\end{align*}
This will conclude the update in Algorithm 2.
\end{proof}

\begin{figure}
\subfigure[$W$ by SGD]{
\includegraphics[width=60mm]{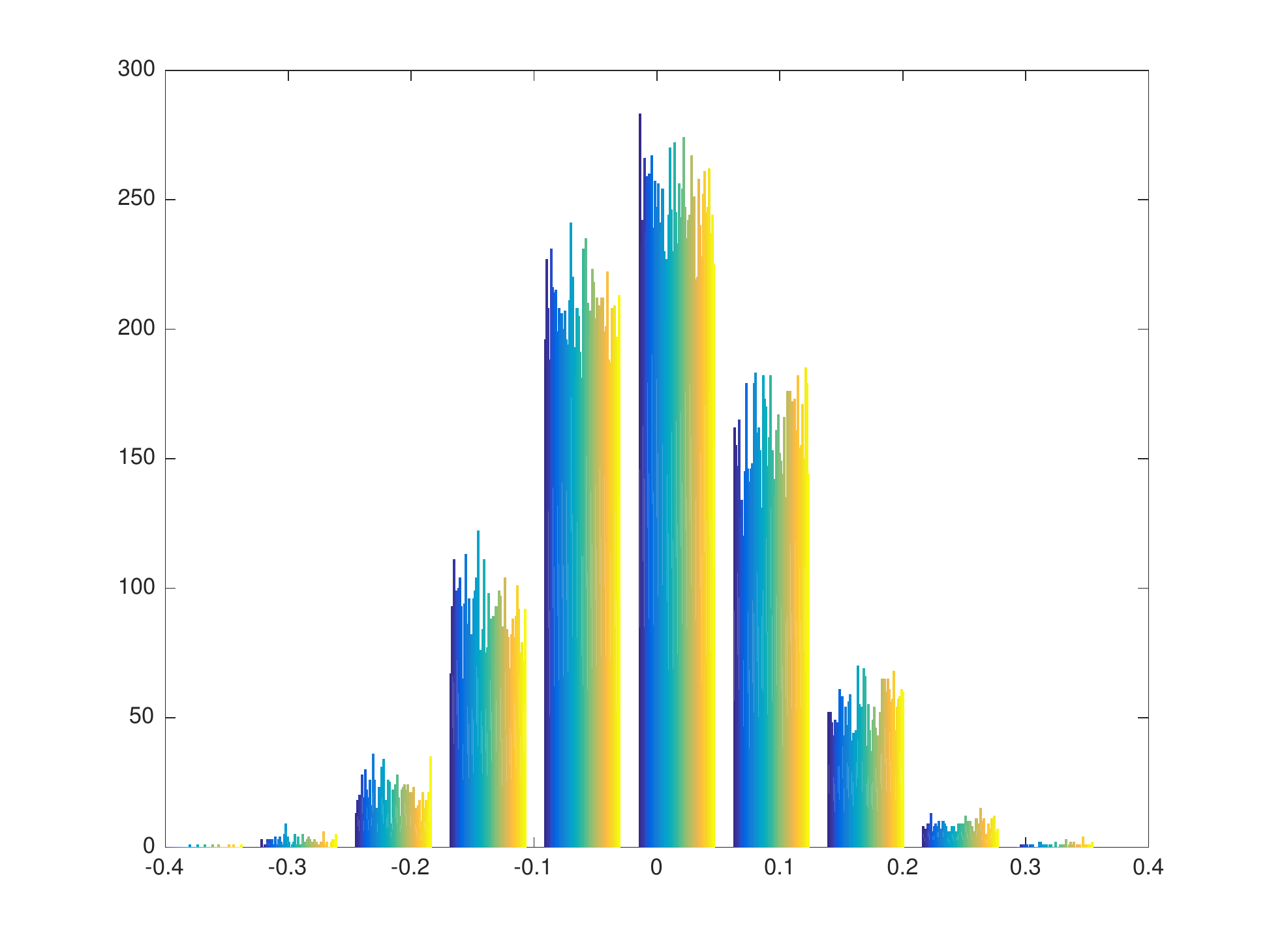}
}
\subfigure[$W$ by SSD]{
\includegraphics[width=60mm]{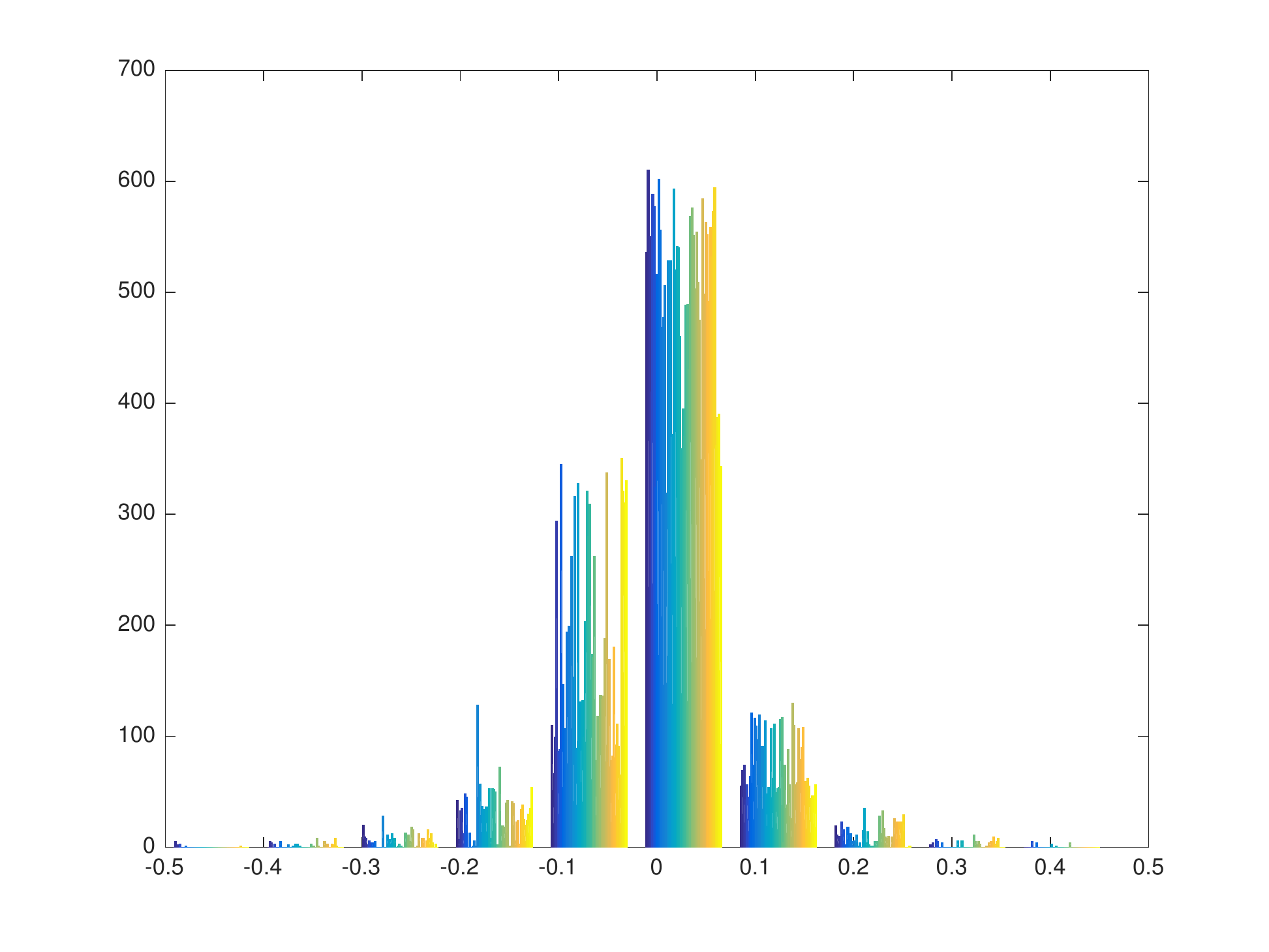}
}
\caption{\textbf{MNIST}: (a-b) Histogram to visualize the empirical distribution of $W$.}
\label{fig:mnist}
\end{figure}

\end{document}